\documentclass[accepted]{uai2021} %
\usepackage[table]{xcolor}

\usepackage[american]{babel}

\usepackage{natbib} %
\usepackage{mathtools} %
\usepackage{paralist}

\usepackage{booktabs,paralist,hhline,colortbl} %

\definecolor{Maroon}{RGB}{204, 102, 0}

\definecolor{rulecolor}{rgb}{0.0, 0.06, 0.54}
\definecolor{tableheadcolor}{rgb}{0.74, 0.83, 0.9}
\definecolor{bluecolor}{rgb}{0.74, 0.83, 0.9}

\usepackage{tikz} %

\usepackage{url}
\usepackage{hyperref}
\usepackage{graphicx}
\usepackage{dsfont}
\usepackage{amssymb}
\usepackage{amsmath}

\usepackage{bbm}
\usepackage{multirow}
\usepackage{rotating}

\renewcommand{\cite}[1]{\citep{#1}}

\usepackage{scalerel, amssymb}
\usepackage{amsmath}
\usepackage{amssymb}
\usepackage{paralist}
\usepackage{array}
\usepackage{listings}
\usepackage{subcaption}
\usepackage{mathtools}
\usepackage{multirow}
\usepackage{comment}
\usepackage{nicematrix,tabularx}
\usepackage{amsthm}

\usepackage[linesnumbered,ruled,vlined]{algorithm2e}
\usepackage{accents}

\newcommand{\undtil}[1]{\underaccent{\tilde}#1}
\def\EE{\mathbb{E}}

\usepackage{ulem}
\newtheorem{theorem}{Theorem}
\newtheorem{definition}{Definition}
\newtheorem{lemma}{Lemma}
\newtheorem{corollary}{Corollary}
\newtheorem{observation}{Observation}
\newtheorem{remark}{Remark}
\newcounter{example}[section]
\newenvironment{example}[1][]{\refstepcounter{example}\par\medskip
   \noindent \textbf{Example~\theexample. #1} \rmfamily}{\medskip}

\title{Graph Reparameterizations for Enabling 1000+ Monte Carlo Iterations \\ in Bayesian Deep Neural Networks}

\author[1,2]{Jurijs Nazarovs}
\author[3,2]{Ronak R. Mehta}
\author[3,2]{Vishnu Suresh Lokhande}
\author[2,3]{Vikas Singh}

\affil[1]{%
    Department of Statistics, 
    University of Wisconsin Madison
}
\affil[2]{%
    Department of Biostatistics \& Med. Info., 
    University of Wisconsin Madison
}
\affil[3]{
     Department of Computer Science,
     University of Wisconsin Madison
}

\begin{document}

\maketitle

\begin{abstract}
    Uncertainty estimation in deep models 
    is essential in many real-world applications and has benefited 
    from developments over the last several years. 
    Recent evidence \cite{farquhar2019radial} suggests that existing solutions dependent on simple Gaussian formulations may not be sufficient.
    However, moving to other distributions necessitates Monte Carlo (MC) sampling to estimate quantities such as the $KL$ divergence: it could be expensive and scales poorly as the dimensions of both the input data and the model grow. 
    This is directly related to the structure of the computation graph, which can grow linearly as a function of the number of MC samples needed.
    Here, we construct a framework to describe these computation graphs, and identify probability families where the graph size can be \textbf{independent} or only weakly dependent on the number of MC samples.
    These families correspond directly to large classes of distributions.
    Empirically, we can run a much larger number of iterations for MC approximations for larger architectures used in 
    computer vision
    with gains in performance measured in confident accuracy, stability of training, memory and training time.
\end{abstract}

\begin{figure}[!t]
\centering
\includegraphics[width=0.7\linewidth]{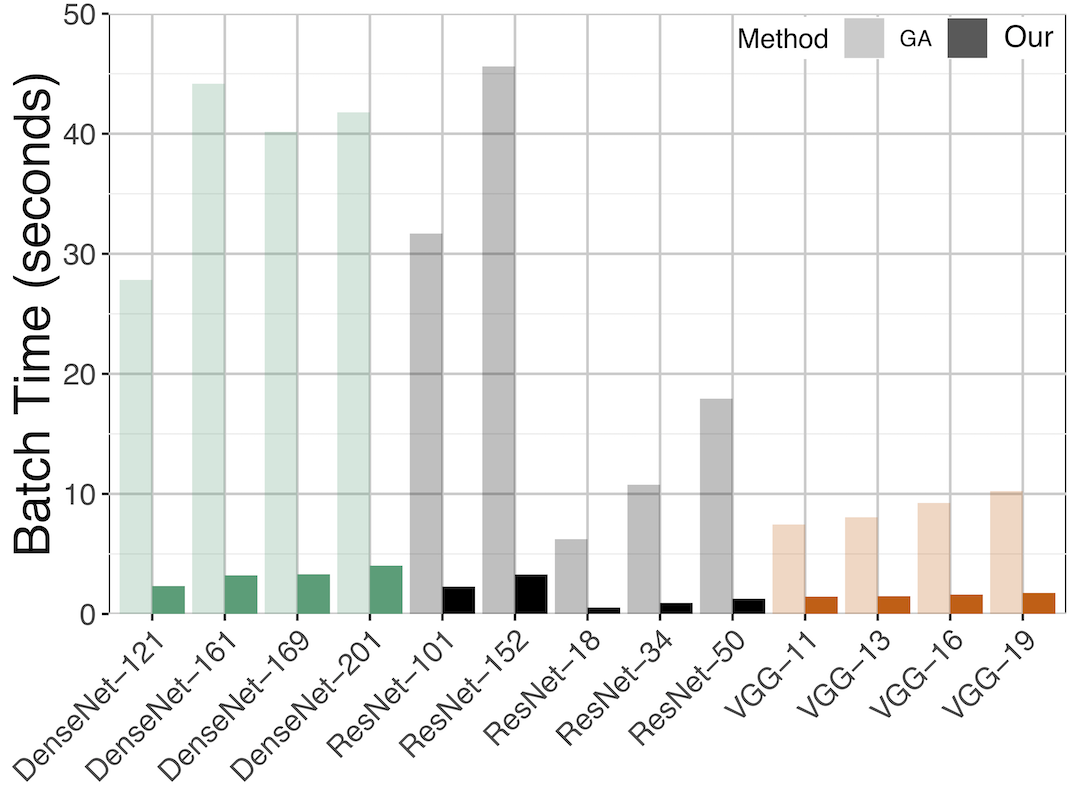}
\vspace{-10pt}
\caption{\footnotesize MC sampling is significantly slower in existing neural network libraries incorporating Gradient Accumulation (GA). In contrast, our proposed MC reparameterization reduces the compute time up to $14\times$ for some networks. }
\label{fig:newfig2}
\end{figure}

\section{Introduction}
\label{sec:intro}

Motivated by the need to provide measures of uncertainty in the deployment of deep neural networks in mission critical and medical applications, 
there has been a strong recent
interest in deep Bayesian learning.
While deep Bayesian learning provides many methods to estimate posterior distributions,
Variational Inference (VI) is a convenient choice for many problem settings \cite{blundell2015weight}. Many libraries such as Tensorflow Probability \cite{dillon2017tensorflow} are also 
now available that offer a rich set of features.

Denote the observed data as $(x, y)$, where $x$ is an input to the network, and $y$ is a corresponding response (in autoencoder settings we may have $y=x$). 
When using VI in Bayesian Neural Networks (BNNs), one considers all weights $W=(W^1, \ldots, W^D)$ as a random vector
and approximates the true unknown posterior distribution $P(W|y, x)$ with an \textit{approximate posterior} distribution $Q_\theta$ of \textit{our choice}, which depends on learned parameters $\theta$. Let $W_\theta=(W_\theta^{1}, \ldots, W_\theta^{D})$ denote a random vector with a distribution $Q_\theta$ and pdf $q_\theta$. VI seeks to find $\theta$ such that $Q_\theta$ is as close as possible to the real (unknown) posterior $P(W|y, x)$, accomplished by minimizing the $KL$ divergence between $Q_\theta$ and $P(W|y, x)$. Given a prior pdf of weights $p$, along with a likelihood term $p(y|W, x)$, and a common \textit{mean field} assumption of independence for $W^d$ and $W_\theta^d,$ for $d \in 1, \ldots, D$, i.e. $p(W) = \prod_{d=1}^D p^d (W^d)$ and $q_\theta (W_\theta) = \prod_{d=1}^D q_\theta^d (W_\theta^d)$, %

\begin{align}
    &\boldsymbol{\theta}^{*} = \underset{\theta}{\arg \min}\ 
    KL\left(q_\theta||p\right) -
    \EE_{q_{\theta}}\left[\ln p(y| W, x)\right]\label{form:vi}\\
    &KL\left(q_\theta||p\right) =\sum_{d=1}^D\EE_{q_{\theta}^d}\left[ \ln q_{\theta}^d(w)\right] -
        \EE_{q_{\theta}^d}\left[\ln p^d(w)\right]\label{form:kl}
\end{align}

\begin{figure*}[t]

\centering
\begin{subfigure}[t]{0.4\linewidth}
    \centering
    \includegraphics[width=\linewidth]{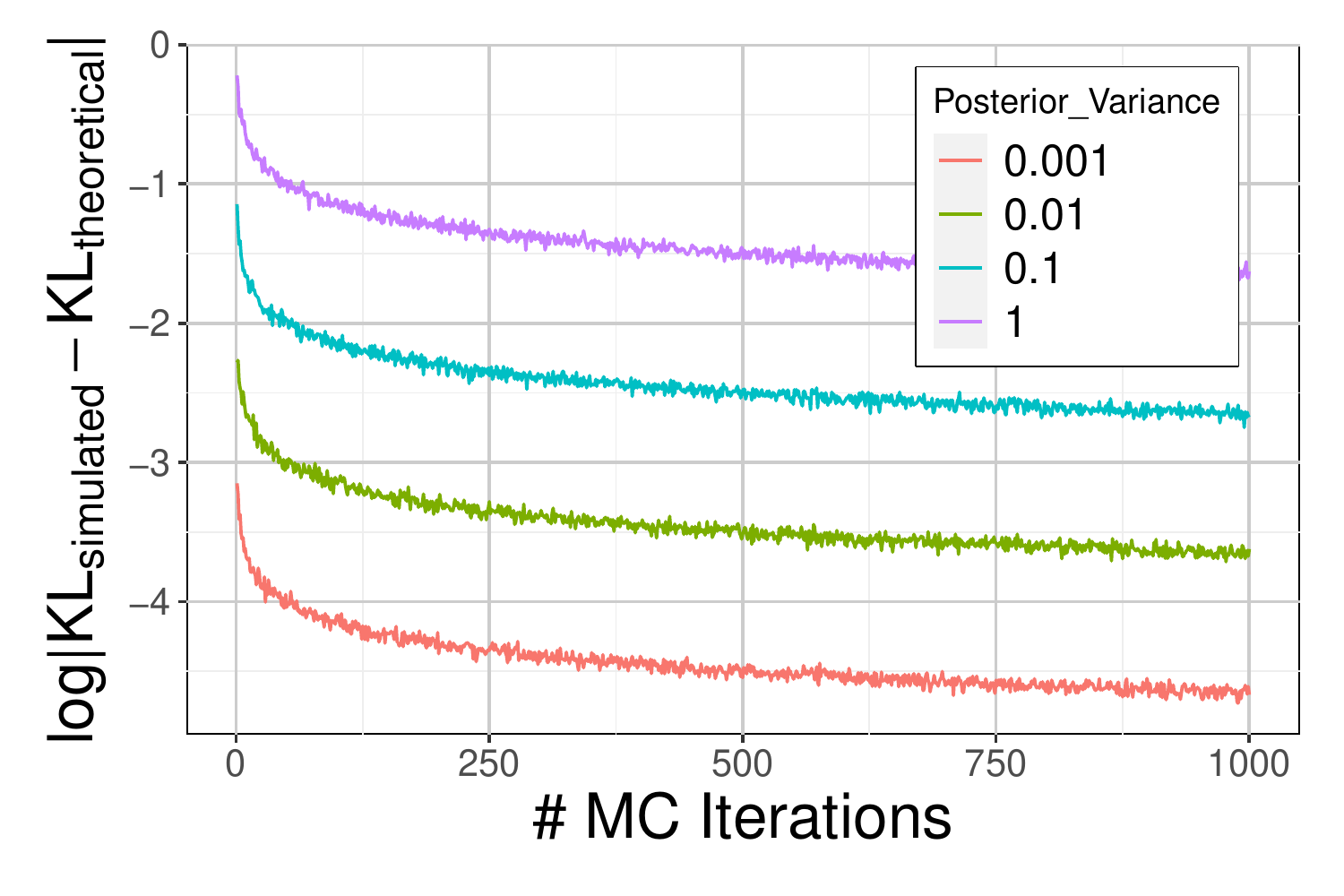}
\end{subfigure}
\hspace{0.5cm}
\begin{subfigure}[t]{0.4\linewidth}
    \centering
    \includegraphics[width=\linewidth]{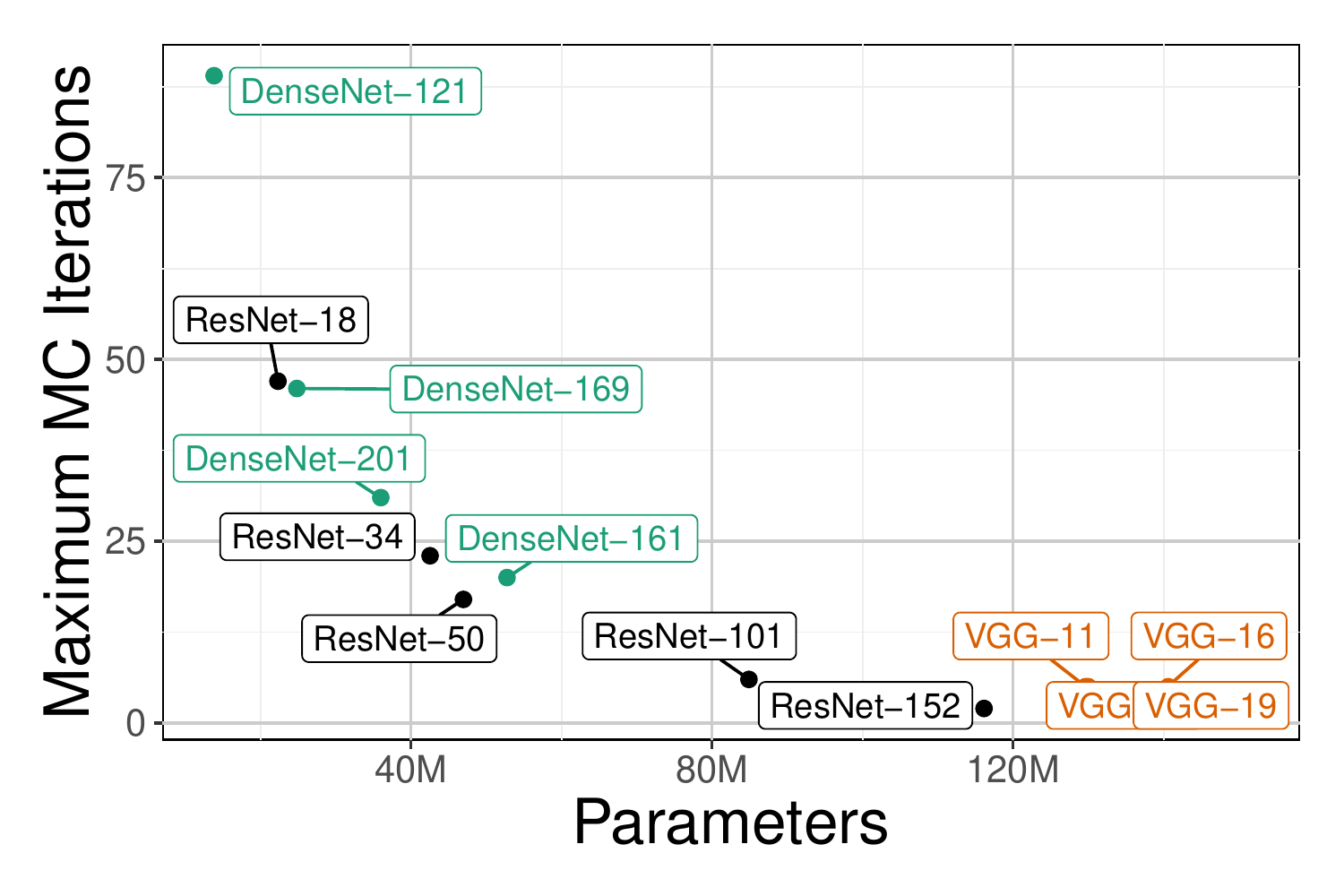} %
    \label{fig:GA}
\end{subfigure}
\vspace{-14pt}
\caption{{\footnotesize \bf(left)} Approximation error ($\log$) of simulated $KL$ for the single-parameter Bayesian Neural Network at different variance values of approximate posterior distribution, {\bf(right)} Maximum number of feasible MC iterations required for training Bayesian versions of different neural networks on a single GPU . } 
\label{fig:klsim}
\end{figure*}

A key consideration in VI is the choice of prior $p$ and the approximate posterior $q_\theta$. This choice 
does not drastically
change 
the computation of the likelihood term $p(y|W, x)$ which 
is influenced more by the problem and the complexity of the network 
instead of $W$ (e.g., it is Gaussian for regression problems). 
But it strongly impacts 
the computation of $KL$ term.
For example, a common choice for $p$, and $q_\theta$ is Gaussian, which allows calculating \eqref{form:kl} in a closed form. However, there is emerging evidence \cite{ farquhar2019radial, fortuin2020bayesian} that the Gaussian assumption may not work well on medium/large scale Bayesian NNs.  \cite{farquhar2019radial} attributes this to the probability mass in high-dimensional Gaussian distributions concentrating in a narrow ``soap-bubble'' far from the mean.
Choosing a correct distribution is an open problem  \cite{ghosh2017model,farquhar2019radial,mcgregor2019stabilising,krishnan2019efficient}, and unfortunately,  more complex distributions frequently lack closed form solutions for \eqref{form:kl}.

{\bf Numerical approximations.}
When the integrals for these expectations cannot be solved in closed form, an approximation is used \cite{ranganath2014black,paisley2012variational,miller2017reducing}. One strategy is Monte Carlo  (MC) sampling, which gives an
unbiased estimator with variance $O(\frac{1}{M})$ where $M$ is number of samples. For a function $g(\cdot)$:
\begin{align}
    \mathbb{E}_{q_\theta}\left[g(w)\right]
    &=
    \int{g(w)q_\theta(w)dw}
    \approx \frac{1}{M}\sum_{i=1}^Mg(w_i),\nonumber\\
    &\text{ where } w_i \sim Q_\theta.
\label{form:mc-intro}
\end{align}
Expected value terms in \eqref{form:kl} can be estimated by applying the scheme in \eqref{form:mc-intro} and in fact, even if a closed form expression can be computed, with enough samples an MC approximation may perform similarly \cite{blundell2015weight}.
Unfortunately, MC procedures are costly, and may need many samples (i.e., iterations) for a good estimation as the model size grows: 
\cite{miller2017reducing} shows this relationship for small networks, and demonstrates that using fewer samples leads to large variances in the approximation.
In general, for deep BNNs, computation of both $KL$ and expectation of log-likelihood requires numerical approximation with MC sampling, but for now, 
we will only focus on the $KL$ term.

\textbf{How does $M$ affect the $KL$ approximation necessary for large scale VI?}
Consider a standard Gaussian distribution for the approximate posterior $q_\theta$ and prior $p$ for the weights of an arbitrary BNN, and also consider an MC approximation of the $KL$ term in (\ref{form:kl}). In this case, we have a closed form solution for $KL$, which allows checking the approximation quality: the gap between the MC approximation $\widehat{KL}$ and the closed form $KL$.

\textbf{(a)}
Figure \ref{fig:klsim} (left) shows this gap %
for different variances of the approximate posterior for 
a BNN. While decreasing the variance of the posterior distribution indeed reduces the variance of an estimator, with such a small variance on weights, the model is essentially deterministic. 
Clearly by increasing $M$, we decrease the error.
However, in current DNNs, increasing the number of MC iterations not only slows down computation, but severely limits GPU memory. 
\textbf{(b)} Figure %
\ref{fig:klsim} (right) presents the maximum number of iterations possible on a single GPU (Nvidia 2080 TI) with a direct implementation of MC approximation for Bayesian versions of popular DNN architectures: ResNet, DenseNet and VGG (more details in \S\ref{sec:pblm}). Extrapolating Figure \ref{fig:klsim}, we see clearly that Bayes versions of these networks will result in large variances. This raises the question: is there a way to increase the number of MC iterations for deep networks without sacrificing performance, memory, or time? 

{\bf Contributions.}
This work makes two contributions. 
\textbf{(a)}
We propose a new framework to construct an MC estimator  for the $KL$ term, which significantly decreases GPU memory needs and improves runtime. Memory savings allow us to run up to 1000$\times$ more MC iterations on a single GPU, resulting in smaller variances of the MC estimators, improving both training convergence and final accuracy, especially on subsets of data where the model is not confident. We show feasibility for popular architectures including ResNets \cite{he2016identity}, DenseNets \cite{huang2017densely}, VGG \cite{simonyan2014very} and U-Net \cite{ronneberger2015u} -- %
strategies for successfully training Bayesian versions of many of these (deep) networks remain limited \cite{dusenberry2020efficient}.
\textbf{(b)}
From the user perspective, we provide a simple interface for implementing and estimating BNNs (Figure.~\ref{fig:code}).
\textbf{(c)}
On the technical side, we obtain a scheme under which we can determine whether our \textit{reparameterization} can be applied. The result covers a broad class of distributions used in VI as an approximate posterior and prior. Inspired by the Pitman–Koopman–Darmois theorem \cite{koopman1936distributions}, we show that our method is effective when an exponential family is used as a prior on weights in deep BNNs estimated via VI, and the approximate posterior is modeled as location-scale or certain other distributions, 
expanding the range of distributions that can be used.

\begin{figure}
\centering
\begin{lstlisting}[language=Python,
        commentstyle=\itshape,
        basicstyle=\ttfamily\smaller,
        frame=lines,
        breakatwhitespace=false,         
        breaklines=true,                 
        keepspaces=true,                 
        showspaces=false,                
        showstringspaces=false,
        showtabs=false,                  
        tabsize=2]
model = AlexNet(n_classes=10, n_channels=3, 
                approx_post="Radial",
                kl_method="repar",
                n_mc_iter=1000)
\end{lstlisting}
\caption{\footnotesize\label{fig:code} 
Proposed MC reparameterization presented as an API. Only a minimal change in an existing programming interface is required to incorporate our method.
See the appendix for details.}
\end{figure}

\begin{figure}[!b]
\centering
\begin{subfigure}[b]{\linewidth}
\centering
\includegraphics[trim={3.5cm 0cm 4cm 0}, height=3.5cm]{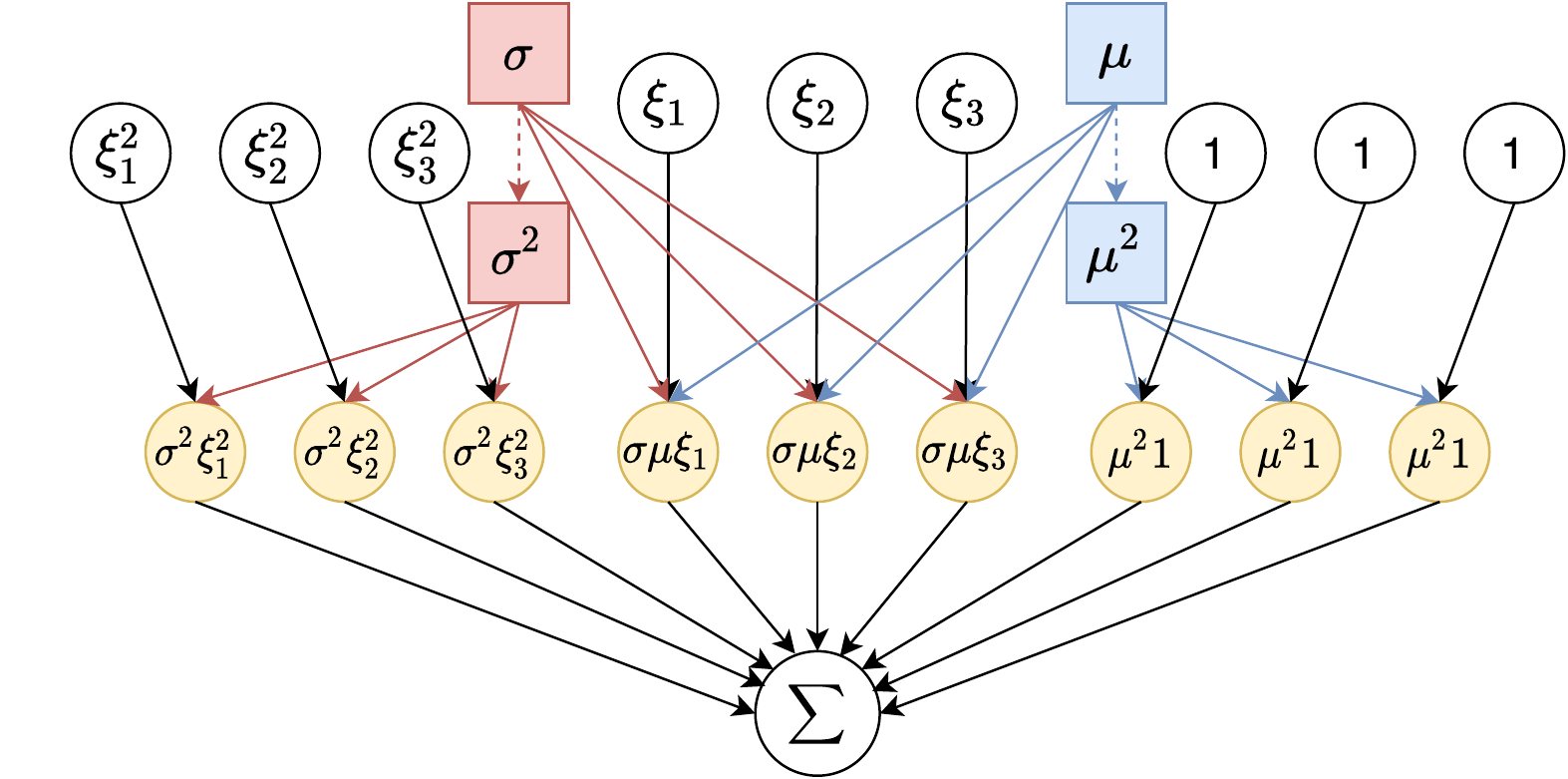} 
\caption{AutoGrad implementation, $d_P = 9$}
\label{fig:compgraph_1}
\end{subfigure}
\vspace{0.1cm}
\begin{subfigure}[b]{\linewidth}
\centering
\includegraphics[trim={2cm 0cm 3cm 0}, height=3.4cm]{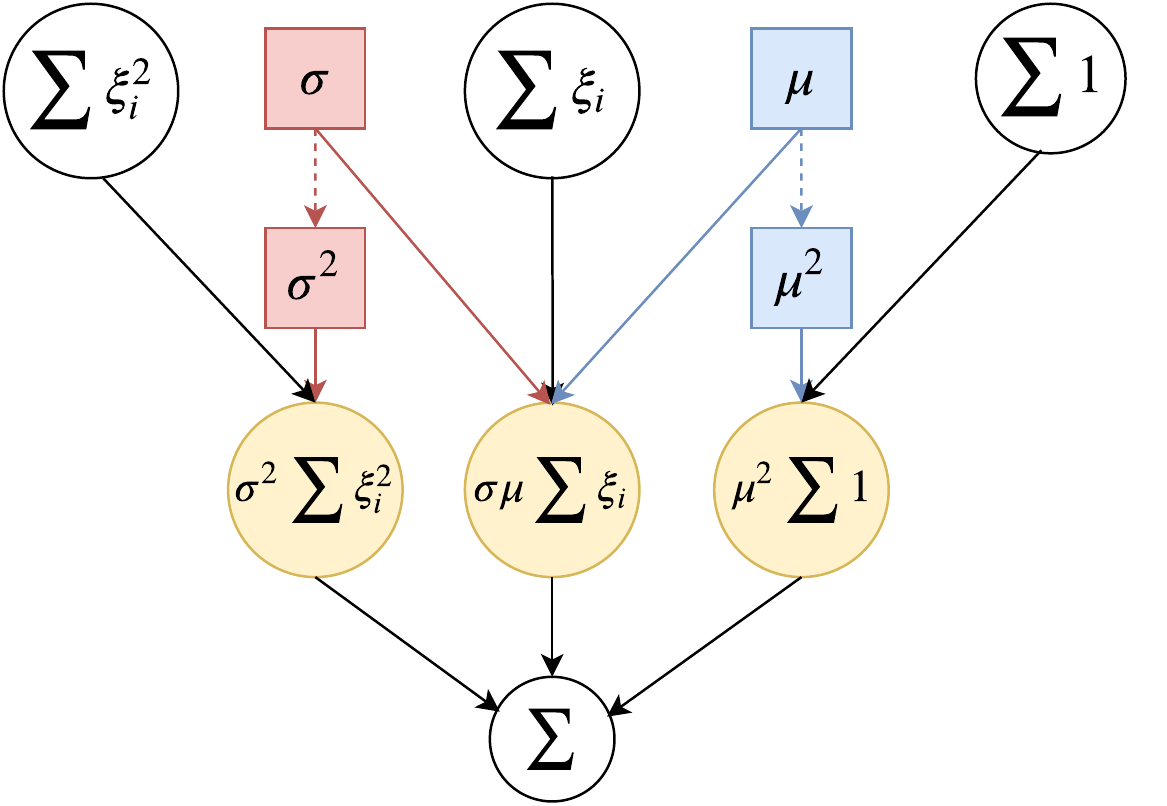}
\caption{$d_P = 3$}
\label{fig:compgraph_3}
\end{subfigure}
\caption{\footnotesize Two computation graphs of the same MC expression $\sum_{i=1}^3 (\mu+\sigma\xi_i)^2$, with different parameterizations. Filled squares represent elements of the vector function $n(\theta)$, clear circles represent functions of auxiliary variables $t(\xi)$, yellow circles represent the Hadamard product $n(\theta)\circ t(\xi)$. Clearly, parameterization affects the size of the graph, and there exists parameterization (b) where the size is independent of the number of MC iterations $M$.
\textbf{Note:} We slightly modify the computational graph presentation for space and clarity.
Actual computation graphs from PyTorch convey the same message. %
}
\label{fig:compgraph}
\end{figure}

\section{Related work}\label{sec:bknd}
In addition to VI, 
the 
literature provides a broad range of ways to estimate posterior predictive distributions. 
Ensemble methods \cite{lakshminarayanan2017simple, pearce2018bayesian, newton2018weighted} can be applied to common networks with minimal modifications; however, they require many forward passes, often similar in terms of space/time to a standard \textit{gradient accumulation} schemes (we provide a PyTorch code snippet in Figure~\ref{fig:ga_code}). 
Figure~\ref{fig:newfig2} provides experimental results showing that gradient accumulation is much slower. %
Other methods like %
Deterministic Variational Inference~\cite{wu2018deterministic} and Probabilistic Backpropagation~\cite{hernandez2015probabilistic}, improve over na\"ive MC implementations of VI, but often approximate the posterior of a neural network with a Gaussian distribution. However, 
\cite{farquhar2019radial} shows that 
Gaussians are sensitive to hyperparameter choices, among other problems during training. For this reason, a non-Gaussian distribution can be used as an approximate posterior in the traditional VI setup, but its lack of a closed form solution for the $KL$ term ends up needing MC approximation. This is where our proposal offers value.
Also, note some other issues that emerge in Deterministic Variational Inference and Probabilistic Backpropagation: (a) the methods need non-trivial modification of the network to perform a moment matching and (b) replacing the Gaussian assumption with another distribution requires new analytical solutions of closed forms. This is more complicated than a MC approximation.

Our work 
is distinct from other works that also target MC estimation in neural networks. 
For example, one may seek to derive new estimators with an explicit goal of variance reduction (e.g., \cite{miller2017reducing}). 
Here, we do {\em not} obtain a new estimator replacing the MC procedure with a smaller variance procedure.
Instead, we study a scheme that makes the computation graph mostly independent of the number of samples, 
and is applicable to ideas such as those in \cite{miller2017reducing} as well. %

\begin{figure}[b]
    \centering

\begin{lstlisting}[language=Python,
        commentstyle=\itshape,
        basicstyle=\ttfamily\small,
        frame=lines,
        breakatwhitespace=false,         
        breaklines=true,                 
        keepspaces=true,                 
        showspaces=false,                
        showstringspaces=false,
        showtabs=false,                  
        tabsize=2]
optimizer.zero_grad()
for _ in range(n_mc_iter):
    output = model(inputs)
    loss = computeLoss(output, targets)
    loss.backward()
optimizer.step()
\end{lstlisting}
    \caption{\footnotesize PyTorch implementation of ``gradient accumulation'' technique, a standard method to collect gradient from several different forward passes. Memory consumption is equivalent to 1 forward pass, but time complexity is proportional to number of forward passes.}
    \label{fig:ga_code}
\end{figure}

\begin{table*}[t]
\centering
\begin{tabular}{m{3.5cm} |cc|ccc}
 \specialrule{1pt}{1pt}{0pt}
\rowcolor{tableheadcolor}
 \multicolumn{1}{c}{Sampling: $W(\theta,\xi)$} & \multicolumn{2}{c}{Approximate Posterior p.d.f. $q_\theta$} & \multicolumn{3}{c}{Prior p.d.f. $p(w)$}\\
\toprule
 \multirow{4}{4cm}{Scaling property family: $W(\theta,\xi)=\theta \xi$ \\and related -- Corollary~\ref{cor:comp_single}}
 & Exponential($\theta$) & Standard Wald($\theta$)& Exponential & Standard Wald & Rayleigh\\
 & Rayleigh($\theta$)    & Weibull($k, \theta$)   & Dirichlet %
 & Chi-squared & Pareto\\
 & Erlang($k, \theta$)   & Gamma($k, \theta$)     & Inverse-Gamma  & Gamma     & Erlang\\
 & Error($a, \theta, c$) & Log-Gamma($k, \theta$) & Log-normal &  Error & Weibull \\
 & Inverse-Gamma($k, \theta$) & & Inverse-Gaussian & Normal & %
 \\
\midrule%
\multirow{3}{4cm}{Location-Scale family: $W(\theta,\xi) = \mu + \sigma \xi$, $\theta =(\mu, \sigma)$}
& Normal($\mu, \sigma$) & Laplace($\mu, \sigma$) & Logistic & Exponential & Normal \\
& Logistic($\mu, \sigma$) & Horseshoe($\mu, \sigma$) & & Laplace & %
\\
& Radial($\mu, \sigma$) & & \multicolumn{3}{c}{Normal variations, e.g., Horseshoe, Radial}\\
\midrule
\multirow{1}{4cm}{Corollary~\ref{cor:comp}}
& Log-Normal($\mu, \sigma$) &  & Dirichlet & Pareto & \\
\bottomrule
\end{tabular}
\caption{ \footnotesize  Summary list of approximate posterior distributions $q_\theta$ and priors $p(w)$, which allows to define a parameterization tuple $P$ for MC estimation, such that $d_P$ is independent of $M$. For every cell in ``Sampling: $W(\theta,\xi)$" we can select any combination of $q_\theta$ and $p(w)$. Reference: Radial \cite{farquhar2019radial}, Horseshoe \cite{ghosh2017model}.
\label{tab:distr}}
\end{table*}

\section{Computation Graphs for MC iterations}\label{sec:pblm}

Despite the ability to approximate the expectation in principle, the minimization in \eqref{form:vi} via \eqref{form:mc-intro} is 
difficult for common architectures, and relies on gradient computations at each iterate. Standard implementations make use of automatic differentiation based on computation graphs \cite{griewank2012invented}.

\textit{Computation graphs} are directed acyclic graphs, where nodes are the inputs/outputs and edges are the operations. If there is a single input to an operation that requires a gradient, its output will also require a gradient. As noted in PyTorch manual (cf. Autograd mechanics), a backward computation is never performed for subgraphs where no nodes require gradients. This allows us to replace such a subgraph with one output node and to define the size of the computation graph as the minimal number of nodes necessary to perform backpropagation: the number of nodes which require gradients. Modern neural networks lead to graphs where the number of nodes range from a few hundred to millions. To define the size of a graph, accounting for the probabilistic nature of the MC approximation, we propose the following construction.

\begin{definition}\label{def:tuple}
Consider $w$ as sampled based on a parameter $\theta$ and an ancillary random variable $\xi$, i.e., $w = W(\theta,\xi)$.
If there exist functions $G$, $n$, and $t$ such that a function $F(w_1, \ldots, w_n)$ can be expressed as $G(n(\theta) \circ t(\xi_1,\ldots, \xi_n))$, then we say $P:=(G,n,t)$ is a \textbf{parameterization tuple} for the function $F$, where $\circ$ is the Hadamard product. Let $d_P$ be the dimension of $n \circ t$, corresponding to the number of nodes requiring gradients with respect to $\theta$.%
\end{definition}
To demonstrate the application of the Def.~\ref{def:tuple}, as an example, consider the computation graph for the MC approximation of the function $g(w)=w^2$ in \eqref{form:mc-intro} and given one weight $W_\theta \sim N(\mu, \sigma^2)$. %
Applying the reparameterization trick: $W_\theta = \mu + \sigma \xi$,  $\xi \sim \rm{N}(0,1)$, the
Python form is, %

\begin{lstlisting}[language=Python, 
        showstringspaces=false,
        formfeed=newpage,
        tabsize=4,
        commentstyle=\itshape,
        basicstyle=\small\ttfamily,
        morekeywords={sampler_normal, torch},
        frame=lines,
        caption={},
        label={code:mc_direct}
        ]
for i in range(M):
    # sample 1 observation from N(0, 1)
    sample = sampler_normal.sample() 
    w = mu * 1 + sigma * sample
    loss += w^2 / M
\end{lstlisting}
The computation graph, 
a function of both the parameters $\theta = (\mu,\sigma)$ and of the auxiliary samples $\xi_1$, $\xi_2$, and $\xi_3$, generated by  PyTorch/AutoGrad for $\mathbf{M=3}$ \textbf{iterations} of this loop is shown in Figure~\ref{fig:compgraph_1}. According to Def.~\ref{def:tuple}, $d_P=9$ and
\begin{align*}
n(\theta) &= (\mu^2,2\sigma\mu,\sigma^2,\mu^2,2\sigma\mu,\sigma^2,\mu^2,2\sigma\mu,\sigma^2), \\
t(\xi) &= (1, \xi_1, \xi_1^2, 1, \xi_2, \xi_2^2, 1, \xi_3, \xi_3^2),\\
G(n(\theta) \circ t(\xi)) &= n_1(\theta)t_1(\xi) + \cdots + n_9(\theta)t_9(\xi)    
\end{align*}
Na\"ively, the graph size grows linearly $O(M)$ with the number of MC iterations, as in the direct implementation (Fig.~\ref{fig:compgraph_1}).
For Bayesian VI in DNNs, this is a problem. We need to perform MC approximations of $KL$ terms at every layer. Also, \cite{miller2017reducing} shows that iterating over a large number of samples $M$ might be important for convergence. This constrains model sizes given limited hardware resources.
One might suspect that a ``for'' loop is a poor way to evaluate this expectation and instead the expression should be vectorized. Indeed, creating a vector of size $M$ and summing it will clearly help runtime. But the loop does {\em not} change the computation graph; all trainable parameters maintain the same corresponding connections to samples, and rapidly exhaust memory.

But graphs for the same function can be constructed differently  (see Fig.~\ref{fig:compgraph_3}). For the right parameterization tuple $P$, we can achieve $d_P=3$. This leads us to,  
\begin{remark}
For computation graph of MC approximation $\sum_{i=1}^Mg(w_i)$ and specific $g$, there exists a parameterization tuple $P=(G, n, t)$, such that $d_P$ is independent of $M$.
\end{remark}

For which class of distributions $Q_\theta$ and functions $g(\cdot)$ can we always construct reparameterizations of the MC estimation \eqref{form:mc-intro}, maintaining the size of the computation graph $d_P$ as {\bf independent} of number of iterations $M$?
We explore this in the next section. 
\section{MC reparameterization enables feasible training}\label{sec:mcpar}

\begin{figure*}[t]
\centering
\begin{subfigure}{0.8\textwidth}
\centering
\includegraphics[height=0.02\textheight,trim={0.8cm 14cm 7.5cm 11cm},clip]{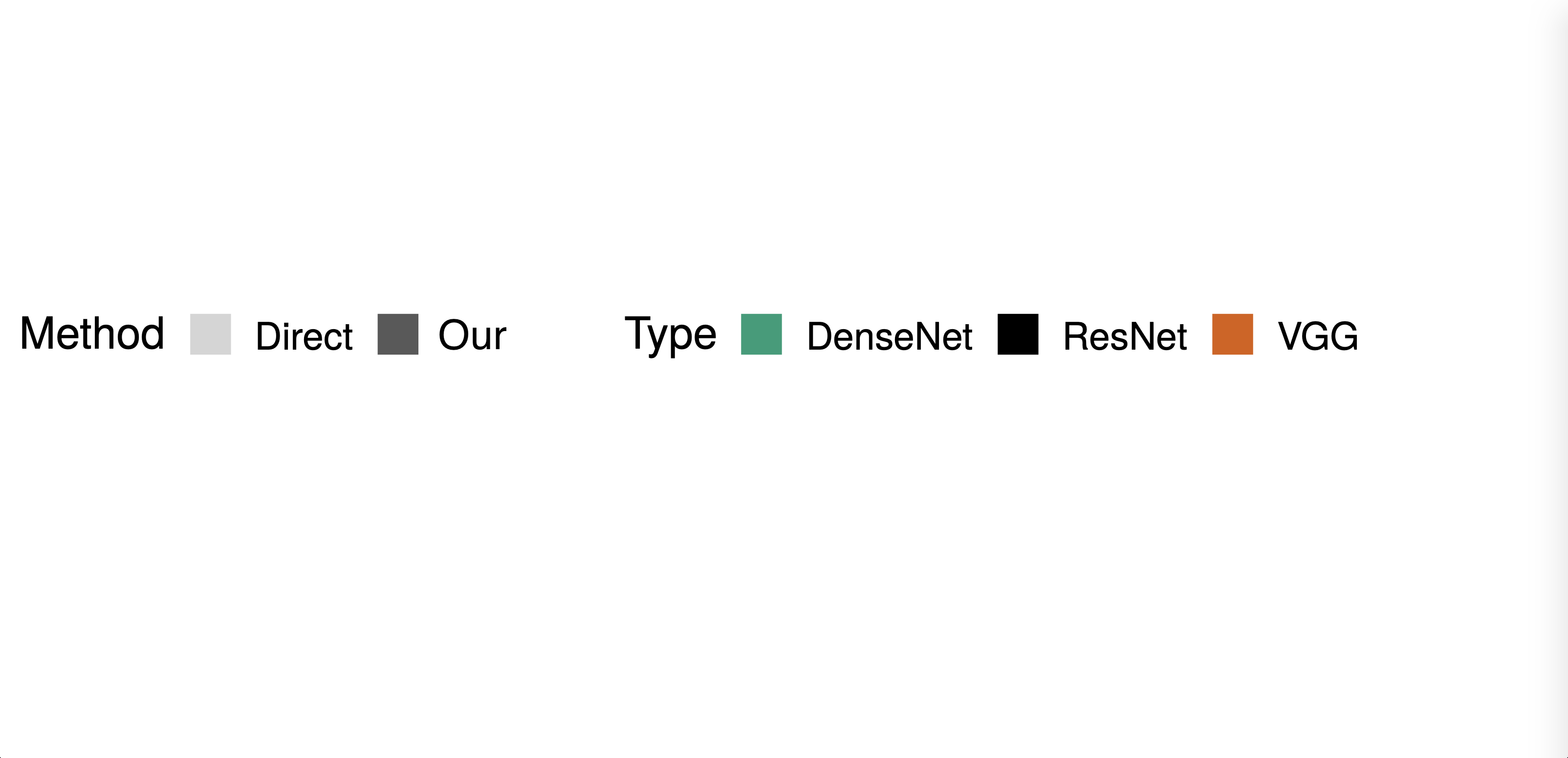} %
\end{subfigure}
\begin{subfigure}{0.4\textwidth}
\centering
\includegraphics[trim={0.3cm 0.3cm 0.3cm 0.3cm}, clip, height=0.2\textheight]{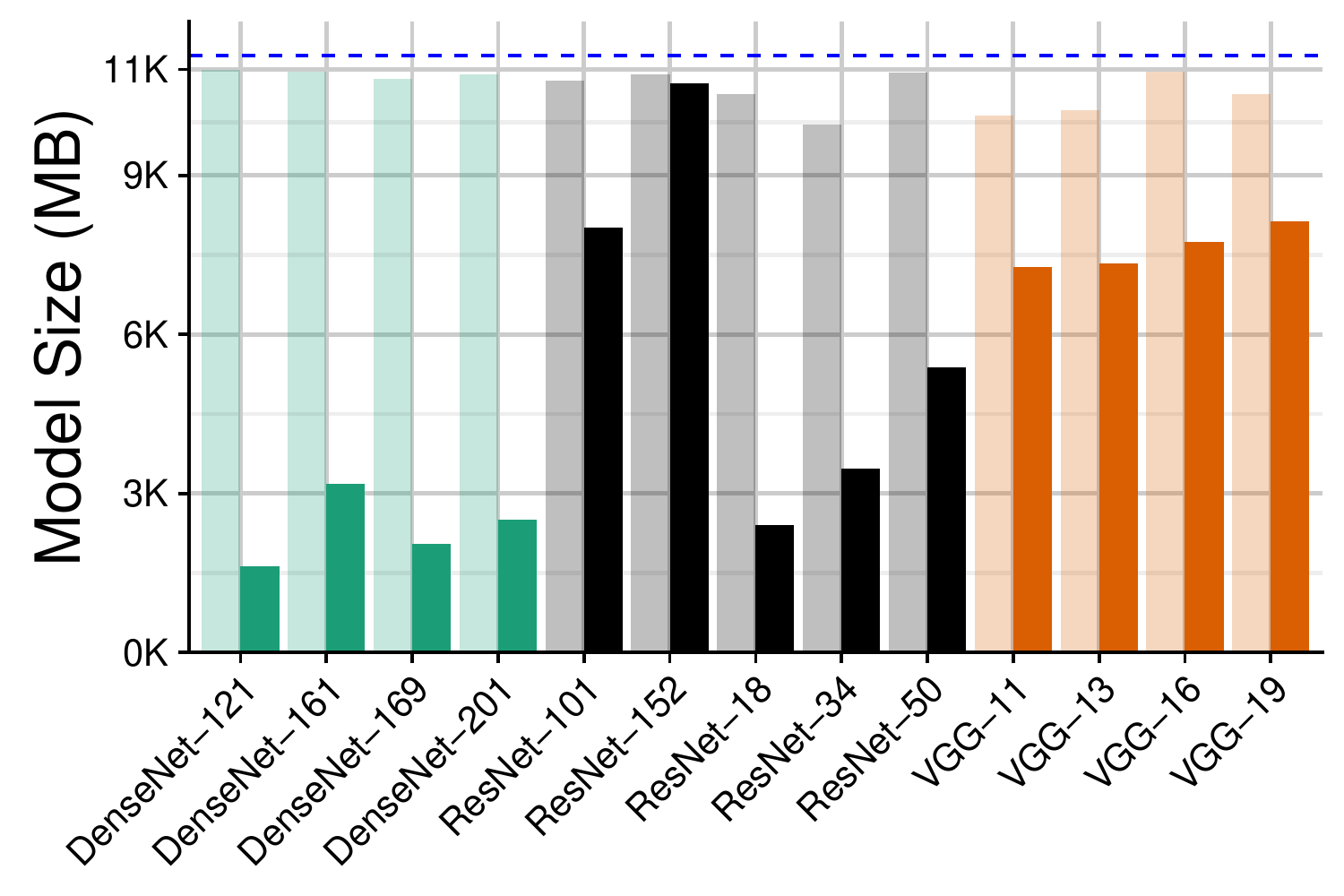}
\caption{\label{fig:space}}
\end{subfigure}
$\quad$
\begin{subfigure}{0.4\textwidth}
\centering
\includegraphics[trim={0.3cm 0.3cm 0.3cm 0.3cm}, clip, height=0.2\textheight]{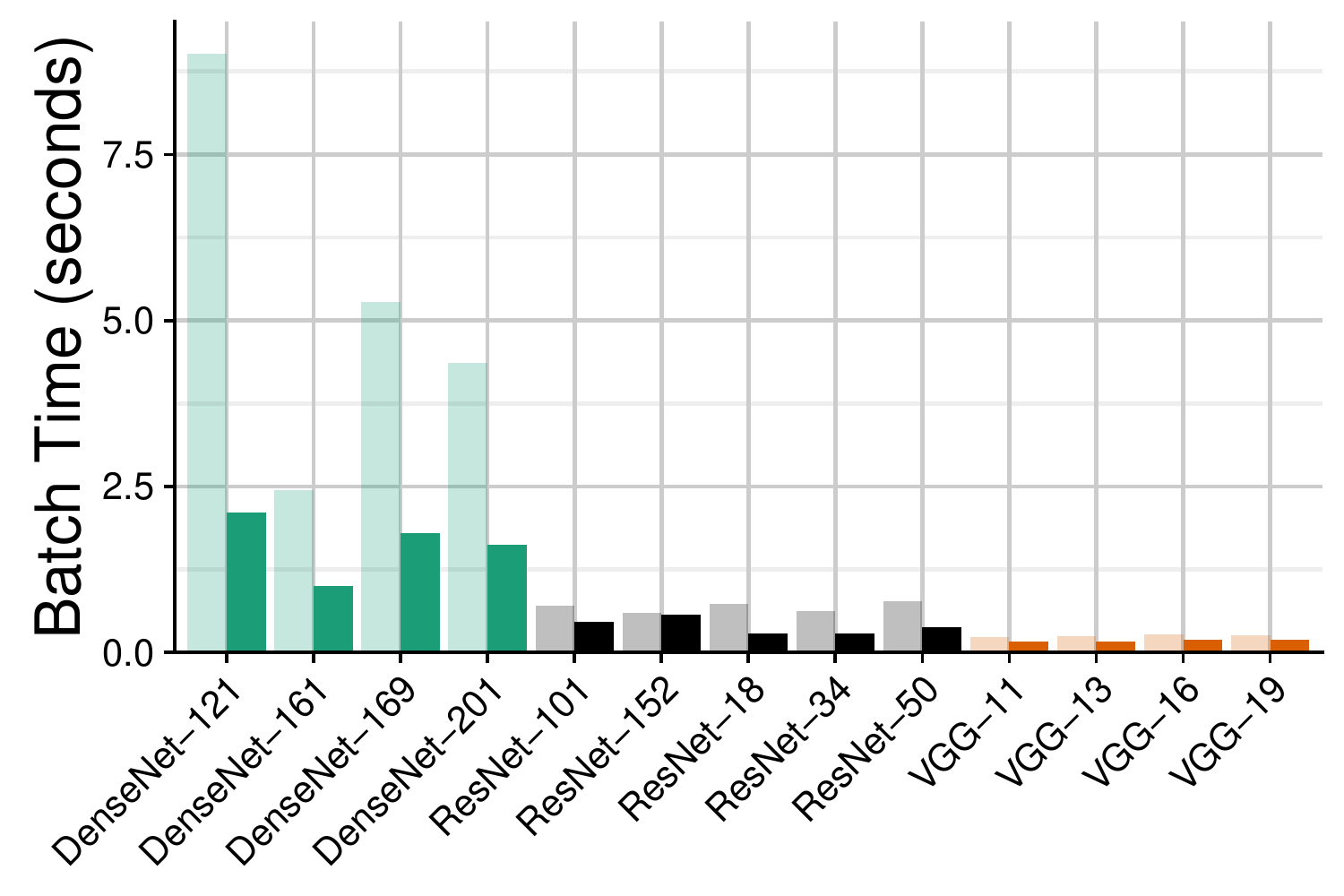}
\caption{\label{fig:Tmax}}
\end{subfigure}
\vspace{-12pt}
\caption{For maximum possible number of MC iterations for a given model via the direct MC method, we show: {\bf(a)} Model size (dashed blue line indicates GPU capacity, 11GB), {\bf(b)} Training time. For some networks, our method occupies less than 25\% of memory and 5 times faster.}
\label{fig:hists}
\end{figure*}

Our approach 
is partly inspired by a vast literature on 
known distributional families and their use within VI. 
For example, in VI, commonly one chooses distributions that fall within \textit{exponential families} (e.g., Gaussian, Laplace, Horseshoe). With this assumption on the prior, we can express 
\begin{equation}
    p(w;\zeta)=h(w){\exp(\eta(\zeta)'}T(w) - A(\zeta))
    \label{form:exp-family}
\end{equation}
where $\zeta$ is a parameter defining $w$.
The sufficient statistics $T(w)$ and natural parameters $\eta(\zeta)$ completely define a specific distribution.

{\bf Relevance of PKD theorem.} 
While the foregoing discussion links our approach to well-known statistical concepts, it does not directly yield our proposed scheme. 
To see this, recall that the Pitman-Koopman-Darmois (PKD) theorem states that for exponential families in \eqref{form:exp-family}, there exist sufficient statistics such that the number of scalar components does not increase as the sample size increases. However, in approximating \eqref{form:kl} with MC, we need to compute not only terms containing the sufficient statistics $T(w)$ {\em but also} $\frac{1}{M}\sum_{i=1}^M \log h(w_i)$. 
Regardless, even though the PKD result cannot be applied directly in our case, it still suggests considering members of the exponential family as candidates for $Q_\theta$. We derive technical results for the forms of $W(\theta,\xi)$ and $g(\cdot)$, where the graph size is not affected by MC sampling.

To approximate $KL$ in \eqref{form:kl}, we need to compute MC estimation \eqref{form:mc-intro} for $g(w)=\log q_\theta(w)$ (or $\log p(w)$). Assume that the factorization form \eqref{form:exp-family} of distributions $q^d_\theta(w)$ (and similarly $p^d(w)$) and recall that the weights of NN are parameterized as 
$w \sim W(\theta, \xi)$. Then, $\EE_\theta\log q_\theta(w)$ is approximated as:
\begin{equation}\label{eq:klapproxexp}
    \frac{1}{M}\sum_{i=1}\left\{\log h(w(\theta, \xi_i)) + \eta(\theta)'T(w(\theta, \xi_i))\right\} - A(\theta)
\end{equation}
To keep the graph size agnostic of $M$, we must handle the initial two terms in \eqref{eq:klapproxexp}. Checking  distributions from Tab.~\ref{tab:distr}, our work reduces to functions of the form $w^k$ and $\log(w)$.

Denote $S$ as the dimension of $\theta$, i.e., number of parameters defining the distribution $Q_\theta$. For example, for the Exponential($\lambda)$: $S=1$ and $\theta = (\lambda)$; for Gaussian($\mu$, $\sigma$): $S=2$ and $\theta = (\mu,\sigma)$. Denote $k$ to be a positive integer.
\begin{theorem}\label{thm:s1}
If $W(\theta,\xi) = \eta(\theta)T(\xi)$ ($S = 1$), then there exists a parameterization tuple $P$ with $d_P = 1$ for the following functions $g(w)$: $w^k$, $\log(w)$, and $\frac{1}{w^k}$.
\end{theorem}

\begin{corollary}\label{cor:comp_single}
If $W(\theta,\xi)' = f(W(\theta,\xi))$ and $W(\theta,\xi)= \eta(\theta) T(\eta)$, then Theorem~\ref{thm:s1} applies to $W(\theta,\xi)'$ and $g(W(\theta,\xi)')$ if $g(w'(w))$ is: $w^k$, $\log(w)$, and $\frac{1}{w^k}$.
\end{corollary}

\begin{theorem}\label{thm:s2}
If $W(\theta,\xi) = \sum_{s=1}^S\eta_s(\theta) T_s(\xi)$, and $g(w) = w^k$, then there exists a parameterization tuple $P$ with 
\begin{align}
d_P = \binom{k + S -1}{S - 1}.
\end{align}
\end{theorem}

\begin{remark}
As long as $d_P < M$, it is possible to create a computation graph of a smaller size by reparameterization, compared to a direct implementation of the MC approximation.
Note that for a small $M$ it is still possible for a parameterization tuple to generate a graph larger than a na\"ive implementation. For example, consider $\sum_{i=1}^M (\mu +\sigma\xi)^2$. When $M=1$, the na\"ive construction would have $d_P = 2, (n = (\mu,\sigma), t = (1, \xi)$, while a ``nicer'' tuple may have $d_P = 3$ independent of $M$ $(n = (\mu^2, 2\mu\sigma,\sigma^2), t = (1, \xi, \xi^2))$.
\end{remark}

\begin{corollary}\label{cor:comp}
If $W(\theta,\xi)' = f(W(\theta,\xi))$ and $W(\theta,\xi)= \sum_{s=1}^S\eta_s(\theta) T_s(\eta)$, where $S\geq 2$, then Theorem~\ref{thm:s2} applies to $W(\theta,\xi)'$ and $g(W(\theta,\xi)')$ if $g(w'(w)) = w^k$. 
\end{corollary}
\paragraph{Relevance of results:} 
(1) Thm.~\ref{thm:s1} can be applied when $W(\theta,\xi)$ represents a distribution with scaling property: any positive real constant times a random variable having this distribution comes from the same distributional family. 
(2) Thm.~\ref{thm:s2} can be applied, when $W(\theta,\xi)$ is a member of the location-scale family.
(3) Corollaries~\ref{cor:comp_single} and ~\ref{cor:comp} are useful when random variables can be presented as a transformation of other distributions, e.g. $\rm{LogNormal}(\mu, \sigma^2)$ can be generated as $\exp(\rm{N(\mu, \sigma^2)})$. 
Table \ref{tab:distr} summarizes the choice of $q_\theta$ and $p$ for Bayesian VI, which lead to the computation graph size $d_P$ being independent of $M$ in MC estimation.

Although Theorem~\ref{thm:s2} does not suggest that there are no nice parameterization tuples for the case where $g(w) = \log w, 1/w^k$, empirically we did not find tuples that allow for $d_P$ to be independent of $M$. But it is interesting to consider an approximation which does allow for this independence.

\subsection{Taylor Approximated Monte Carlo}
Our results extend to the generic polynomial case where $g(w) = p_K(w)$, an arbitrary polynomial of degree $K$:
\begin{corollary}\label{cor:poly}
If $W= \sum_{s=1}^S\eta_s(\theta) T_s(\xi)$, and $g(w) = p_K(w)$, then there exists parameterization tuple $P$, such that for any $M$ iterations
\begin{align}\label{form:dp_polyn}
    d_P = \binom{K+S}{S} -1.
\end{align}
\end{corollary}
So, can we find a parameterization tuple for any $g(w)$ that we can approximate via a polynomial Taylor expansion?
\begin{theorem}
Let $W= \sum_{s=1}^S\eta_s(\theta) T_s(X)$, $S \geq 2$. If an approximation of $g(w)$  uses $K$ Taylor terms, then Cor.~\ref{cor:poly} applies.%
\end{theorem}
\textbf{Practical implications.} 
If one is limited to running a maximum number of MC iterations $M_{max}$,
such an approximation of $g(w)$ allows a tradeoff between accuracy of running just $M_{max}$ iterations for the real $g(w)$ versus approximating $g(w)$ with $K(M_{max})$ terms and running $M \gg M_{max}$ iterations instead, since $d_p$ is independent of $M$. 
This strategy may not work for approximating non-polynomial functions,  
and is a ``fall-back'' that could be used for arbitrary distributions.
\begin{example}
Let $W = \mu + \sigma \xi \implies S=2$ and $g(w) = \log w$, then 
$$
    \begin{aligned}
        \sum_{i=1}^M g(w_i) &=  \sum_{i=1}^M \log(w_i) \approx \sum_{i=1}^M \sum _{k=0}^{K} \frac {1}{k!}(\mu + \sigma \xi_i - 1)^{k}\\
    \end{aligned}
$$
where we take the Taylor expansion of $\log(w)$ around $w=1$. This is clearly a polynomial function of order $K$, and applying Corollary~\ref{cor:poly}, we have $d_P = \frac{1}{2}(K+1)(K+2) - 1$ interactions. For example, if one is able to run just 9 direct MC iterations, it is possible to approximate $g(w)$ with $K = 3$ terms, allowing any number of MC iterations $M$.
\end{example}

\subsection{Applying reparameterization in Bayesian NN}
Recall that training a Bayesian NN via VI requires the approximation of both the $KL$ term and expected value of log-likelihood in \eqref{form:kl}. While it is clear how MC reparameterization can be applied to approximate the $KL$ term, 
what can we say about the likelihood term?
In general, this term cannot be handled 
by the ideas described so far although some practical strategies are possible.  

Usually, estimating the expectation of the likelihood term is based on \cite{kingma2013auto, kingma2015variational}, where for every data item $b$ in the minibatch (of size $B$), one MC sample is selected, which results in $B$ different samples -- in fact, \cite{kingma2013auto} suggests that the number of samples per data item can be set to one if the minibatch size is ``large enough'' which we will discuss more shortly. 
If a large $B$ is feasible, then 
our scheme might not contribute substantially 
in estimating the likelihood term. However, if $B$ is small, 
then our scheme can provide some empirical benefits, described next.

Let $(x, y)$ be the observed data and $(x_b, y_b)$ be the observed $b$-th data point. Let $w$ correspond to the weights of NN with $L$ layers. We can use $w(l,\cdot)$ to index the weights of 
layer $l$. 
Note that we can draw a unique sample of $w$ 
for each data point $b$ which 
we denote as $w(l,b)$. When $M$ samples 
are drawn for $b$, these will be indexed by 
$w_i(l,b)$ for $i=1,\cdots,M$. Notice that $w_1(l, b)$ is the same as $w(l,b)$.
In the forward pass, $u^l_b$ is the output for the $b$-th data point and $u^L_b$ is the output of the last layer for data point $x_b$. %

\begin{observation}[Likelihood form in BNN]\label{obs:like_NN}
Consider the following form for regression and classification tasks,\\ 
\textbf{Regression:}
Consider $y \sim N(u^L, \widehat\sigma)$, where $\widehat\sigma$ is fixed. Then,
\begin{equation*}%
\begin{aligned}[b]
\hspace{-0.1cm} \log p\left(y_{b} \mid w, x_{b}\right)=
&\log \left(\frac{1}{\sqrt{2 \pi}}\exp \left(-\frac{1}{2}\left(y_b-u_b^L\right)^{2}\right)\right)\\
=&\log \frac{1}{\sqrt{2\pi}} -\frac{1}{2}y_b^{2}-y_b u_b^L+\frac{1}{2}(u_b^L)^2.
\end{aligned}
\end{equation*}
{\ } \\ 
\textbf{Classification: }
Consider a binary classification problem. Then, $y \sim \text{Bern}\left(p\right)$, where $p=\frac{1}{1+\exp\left(-u^L\right)}$.
Thus,
\begin{equation*}
\begin{aligned}
\hspace{-0.3cm} \log p\left(y_{b} \mid w, x_{b}\right)=&\log \left(p^{y_b}(1-p)^{1-y_b}\right)\\
=&-\log\left(1+e^{-u_b^L}\right)-(1-y_b)u_b^L \\
=&-\log(2) + \frac{u_b^L}{2} - \frac{\left(u_b^L\right)^2}{4} + O\left(\left(u_b^L\right)^3\right)\\
&-(1-y_b)u_b^L.
\end{aligned}
\end{equation*}
\end{observation}

Based on the above description, let us assume that the final layer output $u^L$ corresponds to a convolution or a fully connected 
layer with no activation function. Then, the log-likelihood term in a regression and classification setup can be expressed as
$$\log p(y_b|w, x_b) = \text{polynomial}(u_b^{L-1}w(L)).$$ 

{\bf SGVB Estimator.}
Following \cite{kingma2013auto}, the $\EE_{q_\theta}\left[\log p(y|w, x)\right]$ 
term for the minibatch (of size $B$) can be written as 
$$
S_1\coloneqq\frac{1}{B}\sum_{b=1}^B \EE_{q_\theta}\left[\log p(y_b|w, x_b)\right].
$$ 
To approximate the expectation, we use $1$ sample $w(\cdot, b)$ for each data point $(x_b, y_b)$, which results in $S_1=\frac{1}{B}\sum_{b=1}^B\log p(y_b|w(\cdot,b), x_b)$. 
Substituting in
$\text{polynomial}(u_b^{L-1}w(L, b))$ into $\log p(y_b|w(\cdot, b), x_b)$ leads to the 
following form for variance 
$V(S_1)$,
\begin{equation}\label{eq:vs1}
\begin{aligned}[b]
\frac{1}{B}\left(\right.V\left(w(L,b)\right)\EE\left[\left(u_b^{L-1}\right)^2\right] +
V(u_b^{L-1})\EE^2\left[w(L,b)\right] \left.\right), 
\end{aligned}
\end{equation}
plus higher order terms 
which decreases as 
$B$ grows. By efficiently evaluating the KL term, we can utilize the memory savings to increase the batch size $B$ and thus, to decrease the variance of $S_1$.

{\bf MC Reparameterization estimator of likelihood.} 
The above strategy is 
practically sufficient.  
However, if $B$ is limited by hardware, 
we can use the memory savings for more MC samples (higher $M$) for 
improving 
the estimate of the log likelihood term. This 
reduces the variance of first term in
\eqref{eq:vs1} by a factor of $M$, but the scheme described is  
restricted to the last layer.

\begin{figure*}[t]
    \centering
    \begin{subfigure}{0.4\textwidth}
    \centering
    \includegraphics[trim={0.1cm 0.2cm 0.2cm 0.1cm}, clip, height=0.15\textheight]{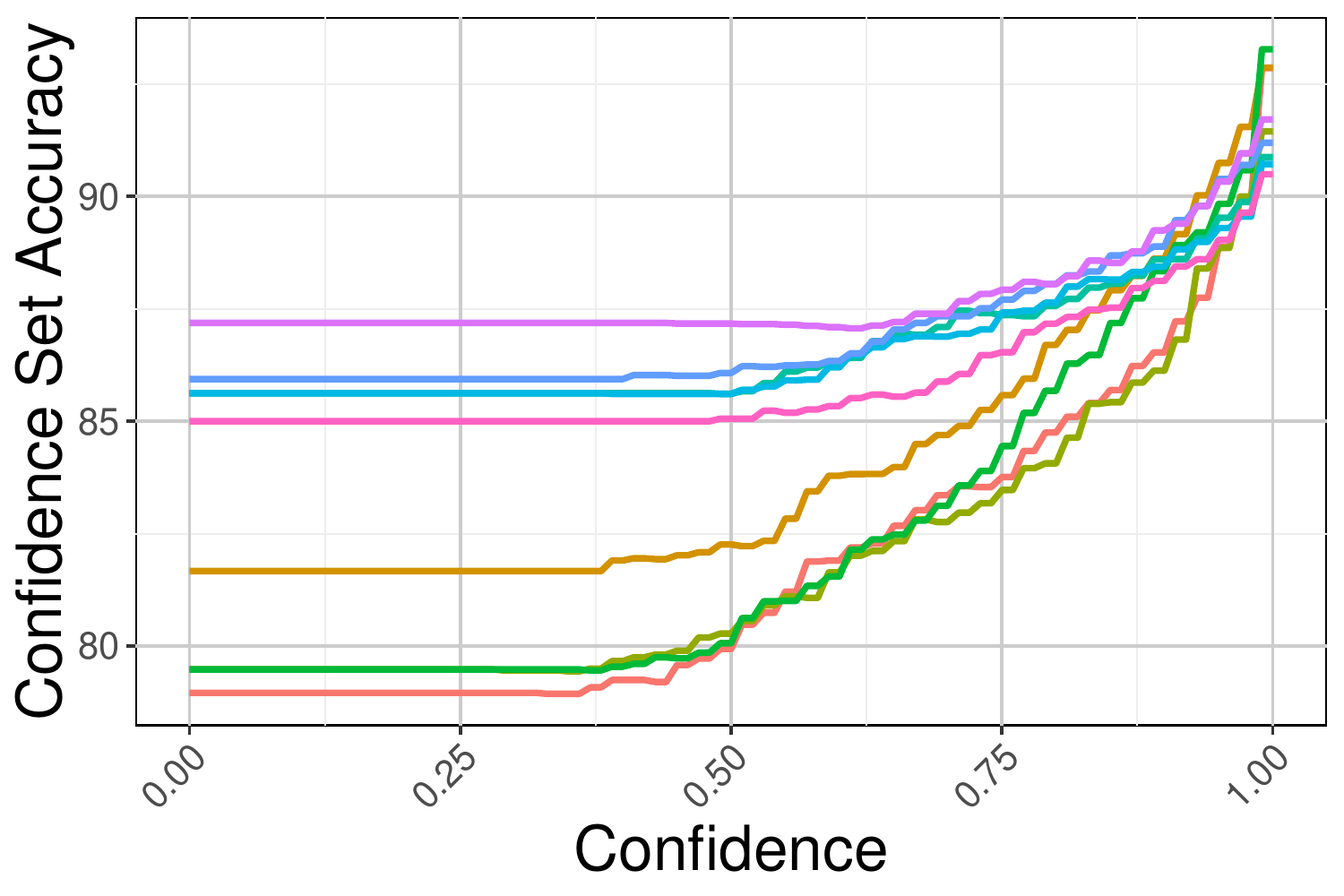}
    \end{subfigure}
    ~~
    \begin{subfigure}{0.13\textwidth}
    \centering
    \includegraphics[trim={0cm 0cm 0cm 1cm}, clip, width=\textwidth]{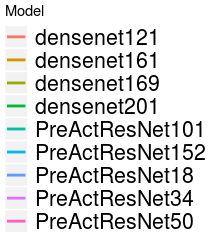}
    \end{subfigure}
    ~~
    \begin{subfigure}{0.4\textwidth}
    \includegraphics[trim={0.1cm 0.2cm 0.1cm 0.1cm}, clip, height=0.15\textheight]{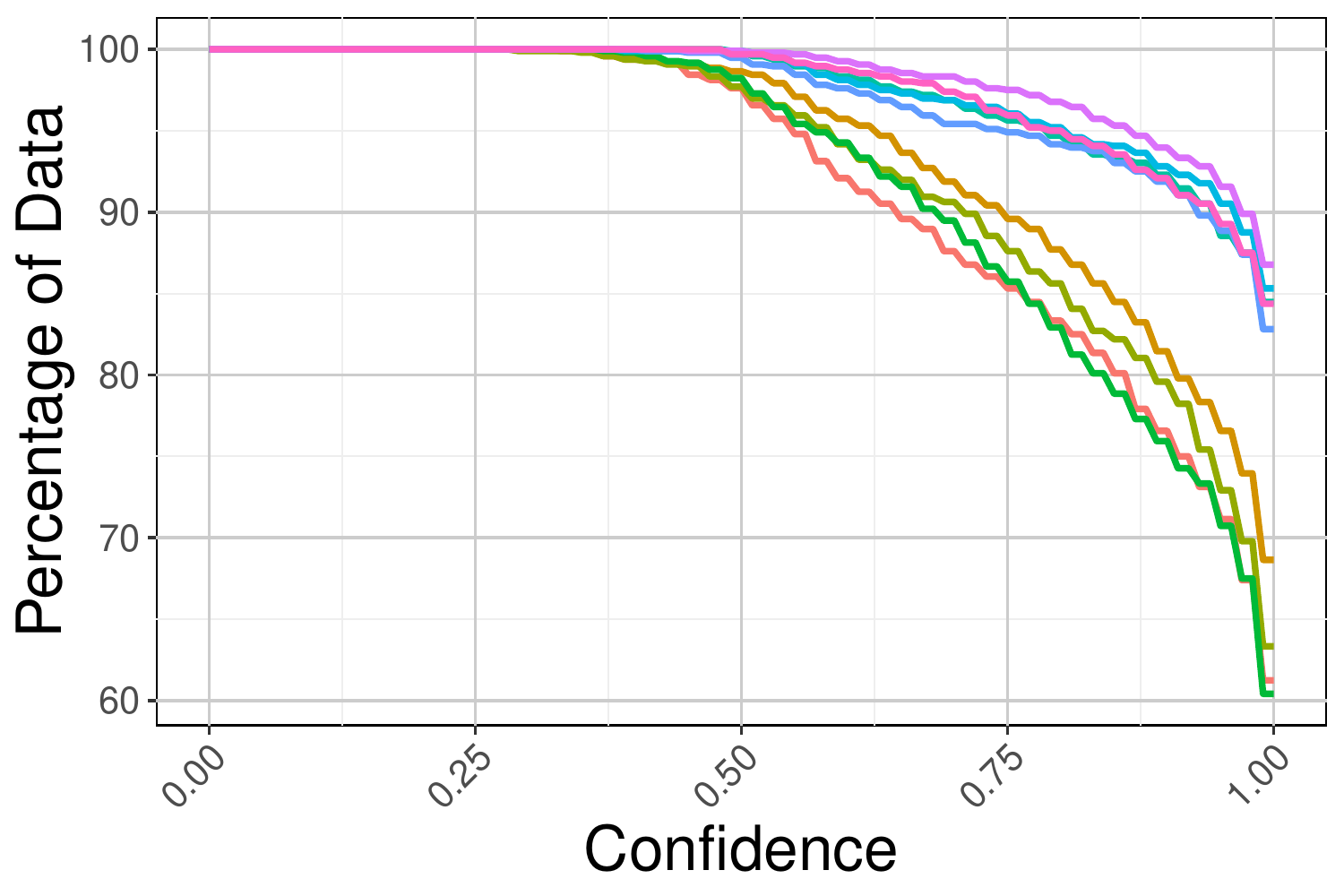}
    \end{subfigure}
    \centering
    \vspace{-10pt}
    \caption{\footnotesize  \label{fig:cifarconf} Confidence Set Accuracy and Confidence Sets on CIFAR-10 for a variety of ResNet and DenseNet models with 100 MC iterations (not previously possible). Both ResNet and Densenet achieve accuracy of more than 90\% with 100\% confidence, but ResNet is 100\% confident on almost 90\% of the data.}
\end{figure*}

\section{Experiments: Bayesian DenseNet, U-Net, and other networks}
We perform experiments on Bayesian forms of several 
architectures and show that training is feasible. 
While we expect some drop in overall accuracy compared 
to a deterministic version of the network, these experiments shed light on the benefits/
limitations of increasing MC iterations.
Since model uncertainty is important 
in scientific applications, we also study the feasibility of training such models for 
classifying high-resolution brain images from a 
public dataset. 

{\bf Setup.} %
For deterministic comparisons, we run several variations of PreActResNet \cite{he2016identity} and Densenet \cite{huang2017densely} (9 in total) on CIFAR10. For brain images, we use a simple modification of 3D U-Net \cite{ronneberger2015u}. Since our method is most relevant
when a closed form for $KL$ is unavailable, we select the approximate posterior to be a Radial distribution, where samples can be generated as: $\mu + \sigma *\frac{\xi}{||\xi||}*|r|$, where $\xi \sim MVN(0, I)$, $r \sim N(0,1)$ and the prior of our weights is a Normal distribution. This satisfies the conditions of Thm.~\ref{thm:s2}, allowing us to find a parameterization tuple that does not grow with respect to $M$: we can run $1000+$ MC iterations with almost no additional GPU memory cost compared to 1 MC iteration. Another reason for choosing the Radial distribution as our approximate posterior $q_\theta$ is because Gaussian-approximate posteriors do not perform well in high-dimensional settings \cite{farquhar2019radial}. Empirically, we find this to be the case as well; we were not able to train any models with a standard Gaussian assumption without any ad-hoc fixes such as pretraining, burn-in, or $KL$-reweighting (common in many implementations). 

{\bf Parameter settings/hardware.} All experiments used Nvidia 2080 TIs. The code was implemented in PyTorch, using the Adam optimizer \cite{kingma2014adam} for all models,
with training data augmented via standard transformations: normalization, random re-cropping, and random flipping.
All models were run for 100 epochs. 

\subsection{Time and Space Considerations}
We first examine whether our MC-reparameterization leads to  
meaningful benefits in model size or runtime. We should expect a competitive advantage in model size as the number of MC iterations grows, which may come at the cost of significantly increased runtime. To allow ease of comparison, we fix the batch size for all models to be 32.
We determine the maximum number of MC iterations able to run on a single GPU for a given model via the classical direct method.
For DenseNet-121, we are able to run 89 MC iterations, while for VGG-16 we are only able to run 5 MC iterations.

Figure \ref{fig:hists} shows a comparison of computational performance between our method and the direct approach.
\begin{inparaenum}[\bfseries (a)]
\item 
With \textbf{our} construction, we significantly \textbf{reduce model size} on the GPU (Fig.~\ref{fig:space}). For smaller models like DenseNet, for the same number of MC iterations our method uses less than 25\% of GPU memory, which allows for a significant \textbf{increase in batch size}. Since the size of the computation graph in our construction is independent of $M$, for the memory used in Fig.~\ref{fig:space} we are able to run for $M=1000$ or more.
\item The significant reduction of model size on the GPU results in a reduction of training time per batch, up to $5\times$ (Figure~\ref{fig:Tmax}); the generated computation graph has fewer parameters (nodes on the path) during backpropagation.
\end{inparaenum}

\subsection{Prediction confidence/accuracy and how many MC iterations?}
For our next set of experiments, we run a Bayesian version of PreActResNet and DenseNet with 100 MC iterations, which is feasible.
\begin{compactenum}[\bfseries (a)]
\item We evaluate the accuracy concurrently with the confidence of the prediction, offered directly by the model.
We expect that the model has a higher accuracy 
for those samples where it highly confident. This is indeed the case -- Figure \ref{fig:cifarconf} shows the accuracy for varying levels of confidence over the entire validation set for a number of models. At high confidence levels, all models perform well, competing strongly with state of the art results. Additionally, we observe the proportion of data for which the model is confident is large (Figure \ref{fig:cifarconf} right). We can see that Bayesian model is at least 75\% confident on 85\%--95\% of data.  
\item One issue in Bayesian networks is evaluating the expected drop in accuracy (compared to its deterministic versions), a behavior common in both shallow and deep models \cite{wenzel2020good}. Figure \ref{fig:cifarconf} (left) reassures us that the drop in performance for a number of widely used architectures is not that significant even when the model is not confident. 
\item To understand the effect of increasing the number of MC iterations,
we run replications of experiment on ResNet-50 for 3 different number of MC iterations, Figure \ref{fig:cifarmc_adni_per_bars}(left):  1 iteration (black), 17 iterations -- maximum possible on GPU with the traditional method -- (blue), and 100 iterations (red) possible to run due to our method. In all cases, as the threshold increases, model confidence increases and as expected, the accuracy does as well. However, we see that training with 100 MC iterations, consistently provides higher accuracy for the entire range of confidences. In contrast, with 1 MC iteration, accuracy has higher variance for the non-confident set. 
\end{compactenum}

\begin{figure}[!b]
    \centering
    \includegraphics[trim={0.3cm 0cm 0cm 0cm}, clip, width=0.49\columnwidth]{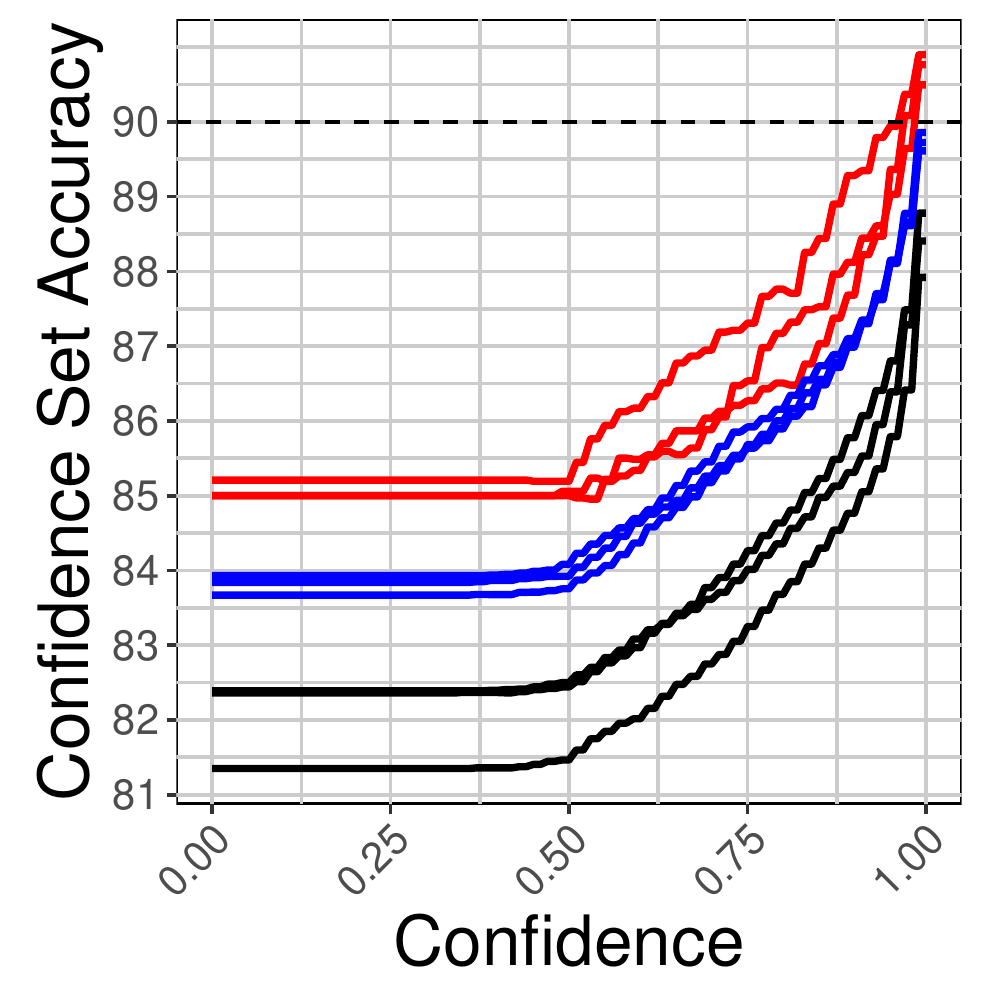}
    \includegraphics[trim={0.2cm 0cm 0.2cm 0cm}, clip, width=0.49\columnwidth]{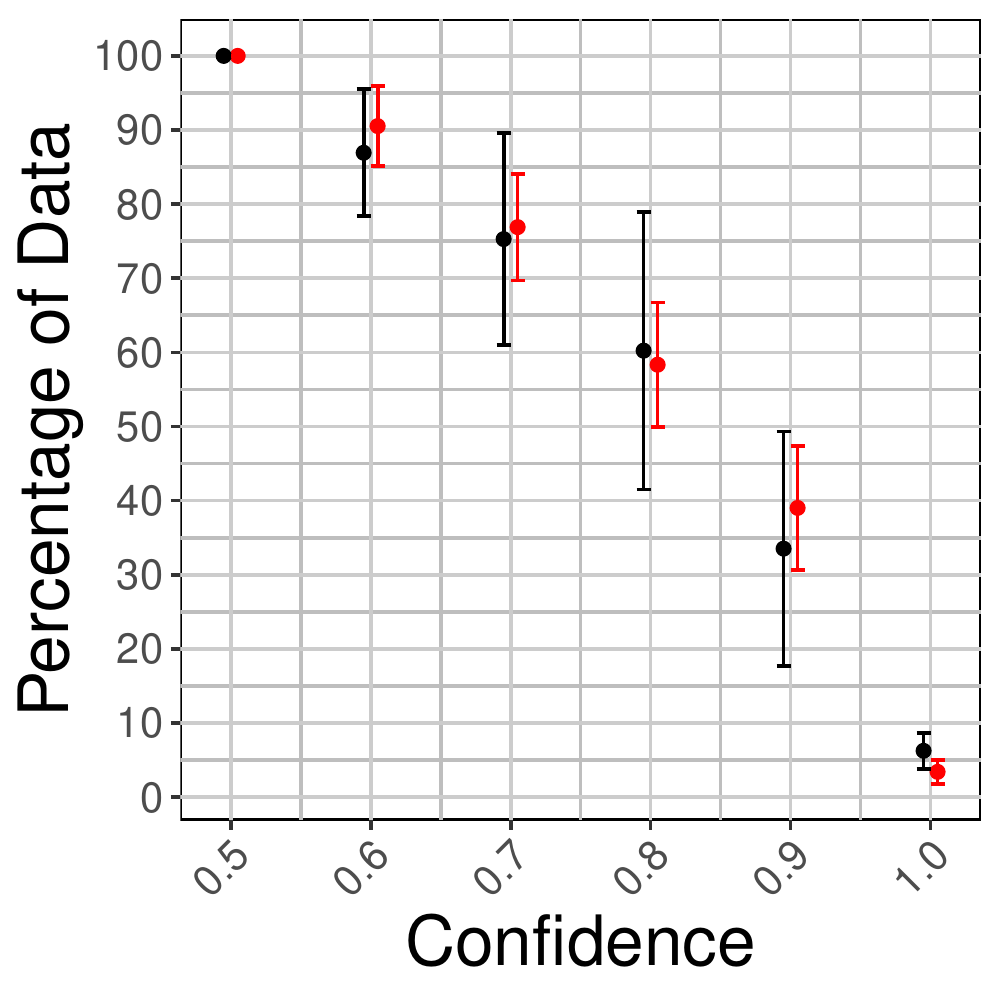}
    \caption{\footnotesize {\bf (left)} Replicated Confidence Set Accuracy on CIFAR-10 for Resnet-50 with different number of MC iterations: 1 (black), 17 (blue, maximum allowed on GPU with standard method) and 100 (red). With $M=100$ the accuracy is higher for any confidence. {\bf (right)} Distributions of Confidence Set Size for a number of replications, with 1 MC iteration (black) and 100 MC iterations (red). With 100 MC iterations variance is smaller.\label{fig:cifarmc_adni_per_bars}}
\end{figure}

\begin{table*}[tbh]
\centering
\begin{tabularx}{0.88\textwidth}{*{7}{c}}
  \specialrule{1pt}{1pt}{0pt}
 \toprule
         \ & \multicolumn{6}{c}{Confidence} \\
         \cmidrule(r){2-7}
 \ & $0.5$ & $0.6$ & $0.7$ & $0.8$ & $0.9$ & $1$\\ 
  \midrule

 $m=1$ & $63.07 \pm 1.47$ & $62.14 \pm 1.59$ & $63.01 \pm 0.64$ & $64.13 \pm 4.15$ & $59.59 \pm 4.40$ & $60.71 \pm 15.15$\\ 
  $m=100$ & $64.39\pm 4.59$ & $66.23 \pm 4.19$ & $66.05 \pm 1.00$ & $67.77 \pm 4.00$ & $66.82 \pm 1.24$ & $87.50 \pm 17.68$\\ 
  
 \midrule
  $\Delta$ & $1.33$ & $4.09$ & $3.04$ & $3.64$ & $7.23$ & $26.79$\\
\bottomrule
\end{tabularx}
\caption{\footnotesize \label{tab:adni_mc} Average validation accuracy per model confidence for 2 values of MC iterations. $\Delta = A_{100} - A_1$, where $A_i$ is validation accuracy for $i$ MC iterations. With 100 MC iterations we got on average much better results, especially when prediction is highly confident.}
\vspace{-10pt}
\end{table*}

\subsection{Neuroimaging: Predictive Uncertainty in Brain Imaging Analysis}

While we demonstrated advantages of 
our reparameterization in traditional image classification settings and benchmarks -- mostly as a proof of feasibility
 -- 
a real need for BNNs is in scientific/biomedical domains: 
where high confidence and accurate predictions may inform diagnosis/treatment. 
To evaluate applicability, we focus on a learning task with brain imaging data. 

{\bf Data.} 
Data used in our experiments were obtained from the 
Alzheimer's Disease Neuroimaging Initiative (ADNI). 
As such, the investigators within the ADNI contributed to the design and implementation of ADNI and/or provided data but did not participate in analysis or writing of this report. A complete listing of ADNI investigators can be found in \cite{adni:authors}. The primary goal of ADNI has been to test whether serial magnetic resonance imaging (MRI), positron emission tomography (PET), other biological markers, and clinical and neuropsychological assessment can be combined to measure the progression of mild cognitive impairment (MCI) and early Alzheimer's disease (AD). For up-to-date information, see \cite{adni:web}.
Classifying healthy and diseased individuals via their MR images, similar to ADNI, is common in the literature  
However, over-fitting 
when using deep models 
remains an issue for two reasons: small dataset size and a large feature space. Here, we look at a specific setting where we have $388$ individuals with pre-processed MR images of size $105 \times 127 \times 105$. {\em Preprocessing.} All MR images were registered to MNI space using SPM12 with default settings.

\begin{figure}[b]
    \centering
\begin{lstlisting}[language=Python,
        commentstyle=\itshape,
        basicstyle=\ttfamily\small,
        frame=lines,
        breakatwhitespace=false,         
        breaklines=true,                 
        keepspaces=true,                 
        showspaces=false,                
        showstringspaces=false,
        showtabs=false,                  
        tabsize=2]
Conv3D_Block(1, 16)
MaxPool3d((3,3,3))
Conv3D_Block(16, 32, stride=1)
MaxPool3d((2,2,2))
Conv3D_Block(32, 64, stride=1)
MaxPool3d((2,2,2))
Conv3D_Block(64, 128, stride=3),
MaxPool3d((2,2,2))
Conv3D_Block(128, 256, stride=3)
Linear(256, 2)    
\end{lstlisting}
    \caption{\footnotesize Structure of the model we used for ADNI classification.}
    \label{fig:unet}
\end{figure}

{\bf Network.} We use a slightly modified version of the encoder from an off-the-shelf 3D U-Net architecture \cite{ronneberger2015u}, demonstrated in Figure~\ref{fig:unet}, to learn a classifier for cognitively normal (CN) and Alzheimer's Disease (AD) subjects. We note that while this architecture is not competitive with those which achieve state-of-the-art classification accuracy on ADNI, our aim here is to demonstrate feasibility of 
training deep Bayesian models in this setting and evaluate the value of accurate confidence estimation.

We train the model on $300$ individuals, and validate on the remaining $88$. Additional experimental details can be found in the appendix. Since the input to the network is a mini-batch of high dimensional images, when we take into account the memory already needed by a deterministic model, we already reach the limits of the GPU memory. While we cannot perform more than 1 MC iteration with the standard method, we can successfully perform more than 100 with our scheme.
We evaluate the consistency of performance with several runs of training when 
we are allowed to use 1 versus 100 MC iterations.
\textbf{(a)} Table \ref{tab:adni_mc} shows the average validation accuracy for the choice of MC iterations and their difference. We see that for every confidence threshold, training with 100 MC iterations provides higher accuracy on average.
This is especially noticeable on a high confident set, where the difference approaches 26.7\%. \textbf{(b)} In addition to accuracy, it is important to understand 
how consistent the estimation is.
Figure \ref{fig:cifarmc_adni_per_bars} (right) demonstrates the distribution of the size of confident set.
While on average, the size of the ``confident set'' of the two models is similar,
the variance is significantly smaller when we use a larger number of MC iterations, consistent with our hypothesis in \S\ref{sec:intro}. 
In cases where this confidence needs to be measured as accurately as possible, one obtains benefits over a single MC iteration.

\section{Conclusions}
While a broad variety of neural network architectures are used in 
vision and medical imaging, successfully training them in a Bayesian setting poses challenges. 
Part of the reason has to do with distributional assumptions.
Moving to a broader class of distributions involves MC estimations but direct implementations pose serious demands on memory and run-time. 
In this work, we identify that different computation graphs can be constructed for different parameterizations of the target function. Specifically when one is attempting a Monte Carlo approximation, these graphs can grow linearly with the number of MC iterations needed, which is undesirable. By directly characterizing the parameterizations that lead to different graphs, we analyze situations where it is possible for graphs to be constructed independent of this sampling rate (number of MC iterations).
Evaluating our parameterization empirically, we find that it is  feasible to run a large number of MC iterations for large networks in vision, with 
a nominal drop in accuracy (compared to deterministic versions). The code 
is available at 
\url{https://github.com/vsingh-group/mcrepar}.

\section*{Acknowledgments}

This work was supported in part by 
NIH grants RF1 AG059312 and RF1	AG062336. 
RRM was supported in part by 
NIH Bio-Data Science Training Program 
T32	LM012413 grant to the University of Wisconsin Madison.

\clearpage

\bibliography{refs}

\newpage
\appendix

\twocolumn[{%
 \centering
 \LARGE APPENDIX
 \vspace{10pt}
}]

In this document we provide more details about experiments, introduce our interactive application to analyze the quality of $KL$ approximation, and give examples of computation graphs of $KL$ terms for different distributions, in comparison between direct implementation and our parameterization technique. Proofs of the results in the main paper can also be found towards the end of the document.

\section{Experiments Details}
A working version of the code is attached in the directory ``main\_code''. In our experiments, we follow the re-weighting scheme for mini-batches proposed by \cite{graves2011practical} as $\beta=\frac{1}{B}$, where $B$ is number of mini-batches.
For all experiments with VGG, we decrease the number of nodes by half in the last dense layers to fit the Bayesian model on a single GPU. We choose an exponential family with 2 parameters, which results in doubling the number of parameters compared to the original networks.

\section{Making your own Bayesian network, using our API}
Figure~\ref{code:own_network} provides an example of how to implement your own Bayesian neural network with our API. 
\begin{figure*}[h]
\begin{lstlisting}[language=Python,
        commentstyle=\itshape,
        basicstyle=\ttfamily\smaller,
        frame=lines,
        breakatwhitespace=false,         
        breaklines=true,                 
        keepspaces=true,                 
        showspaces=false,                
        showstringspaces=false,
        showtabs=false,                  
        tabsize=2]
import bayes_layers as bl
class AlexNet(nn.Module):
    def __init__(self, num_classes, in_channels, 
                 **bayes_args):
        super(AlexNet, self).__init__()
        self.conv1 = bl.Conv2d(in_channels, 64,
                               kernel_size=11, stride=4,
                               padding=5,
                               **bayes_args)
        self.classifier = bl.Linear(1*1*128,
                                    num_classes,
                                    **bayes_args)
        ...
        
    def forward(self, x):
        kl = 0
        for layer in self.layers:
            tmp = layer(x)
            if isinstance(tmp, tuple):
                x, kl_ = tmp
                kl += kl_
            else:
                x = tmp

        x = x.view(x.size(0), -1)
        logits, _kl = self.classifier.forward(x)
        kl += _kl
\end{lstlisting}
\caption{An example of how to implement your own version of the bayesian neural network, using our API. We need to import bayesian layers module, which provides new functional for convolution1d, convolution2d, convolution3d and fully connected layers. In addition we need to redefine forward function, as shown.}
\label{code:own_network}
\end{figure*}

\section{Computation graphs}
In this section we demonstrate computation graphs corresponding to the MC estimation of one of the expectation terms in $KL$ (sometimes it is called $KL$ cross-entropy): $E_{Q_\theta}\log p(w)$, where $Q_\theta$ is the approximate posterior distribution with pdf $q_\theta$, and $p(w)$ is the prior distribution on $w$. We compare the size of computation graphs for different numbers of MC iterations for a direct implementation and our reparameterization method. 

\subsection{Approximate posterior: Radial($\mu$, $\sigma^2$); Prior: Gaussian(0, 1)}
For the following setup there is no closed form solution for $KL$ term, and approximation with MC sampling is required.
Samples from approximate posterior $Q_\theta$ can be generated as $\mu + \sigma\xi$, where $\xi = \frac{w}{|w|}r$, $w \sim \text{MVN}(0, I)$, $r\sim N(0, 1)$. Assumption about the prior gives us the following term to estimate $E_{Q_\theta}\log (\exp(-w^2))$. Figure \ref{fig:rad_norm_cg} shows computation graphs which correspond to different number of MC iterations. We can see that with the direct implementation, the size is proportional to the number of MC iterations, while our approach constructs a graph whose size is independent of the number of MC iterations.

\begin{figure*}[h]
    \centering
    
    \begin{subfigure}{0.4\textwidth}
    \centering
    \includegraphics[width=0.4\linewidth]{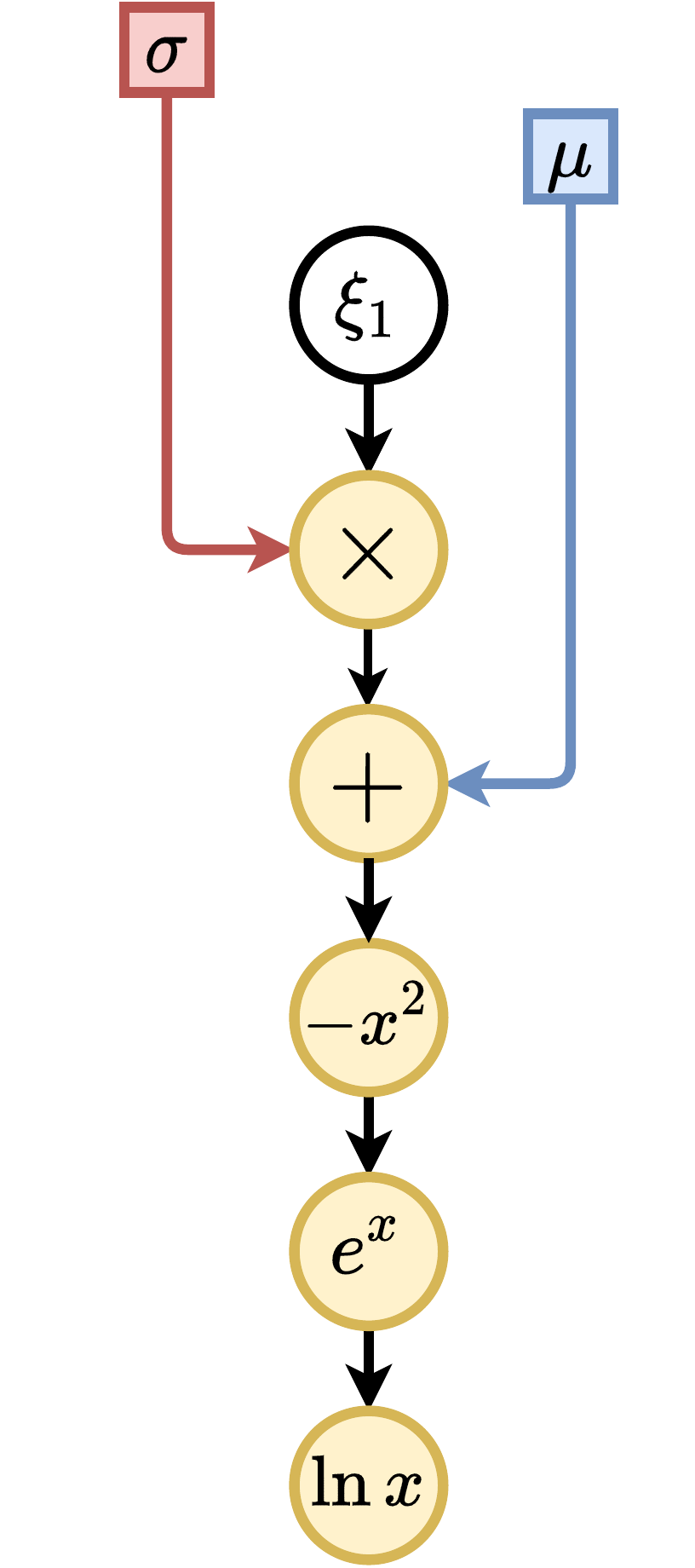}
    \vspace{3pt}
    \caption{$M = 1$}
    \end{subfigure}
    \begin{subfigure}{0.4\textwidth}
    \centering
    \includegraphics[width=0.85\linewidth]{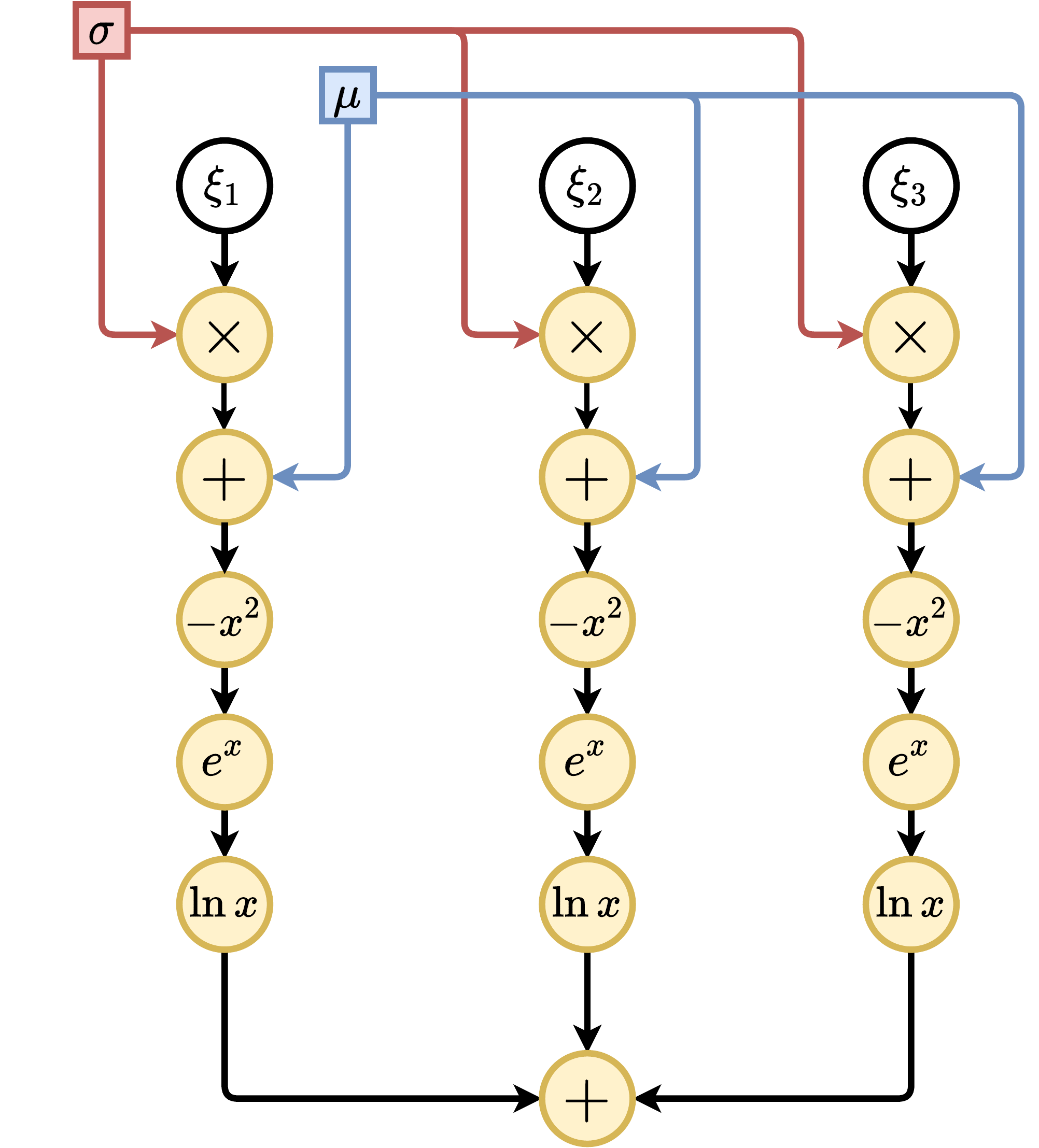}
    \caption{$M = 3$}
    \end{subfigure}
    
    \begin{subfigure}{0.4\textwidth}
    \centering
    \includegraphics[width=0.75\linewidth]{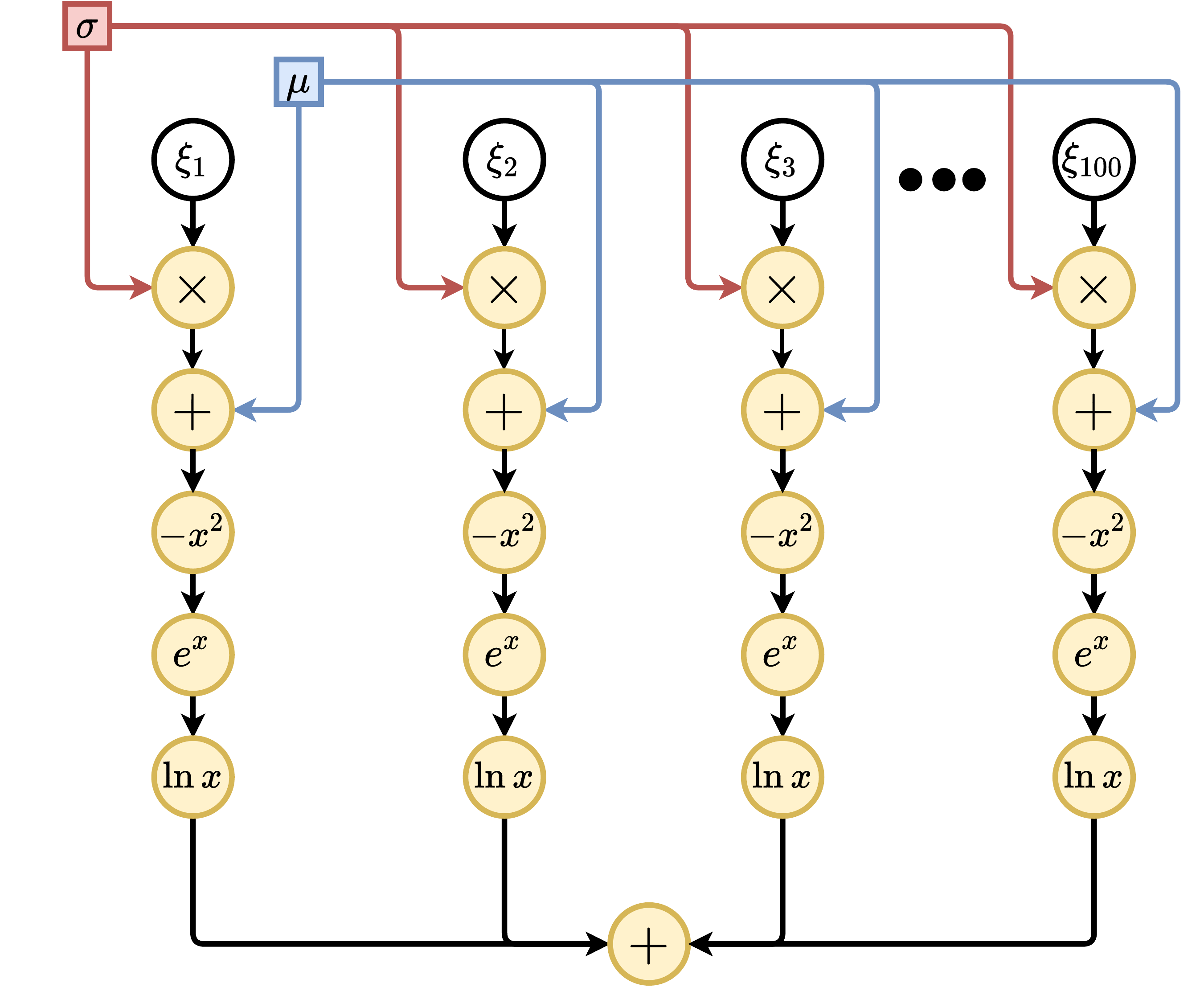}
    \vspace{2pt}
    \caption{$M = 100$}
    \end{subfigure}
    \begin{subfigure}{0.4\textwidth}
    \centering
    \includegraphics[width=\linewidth]{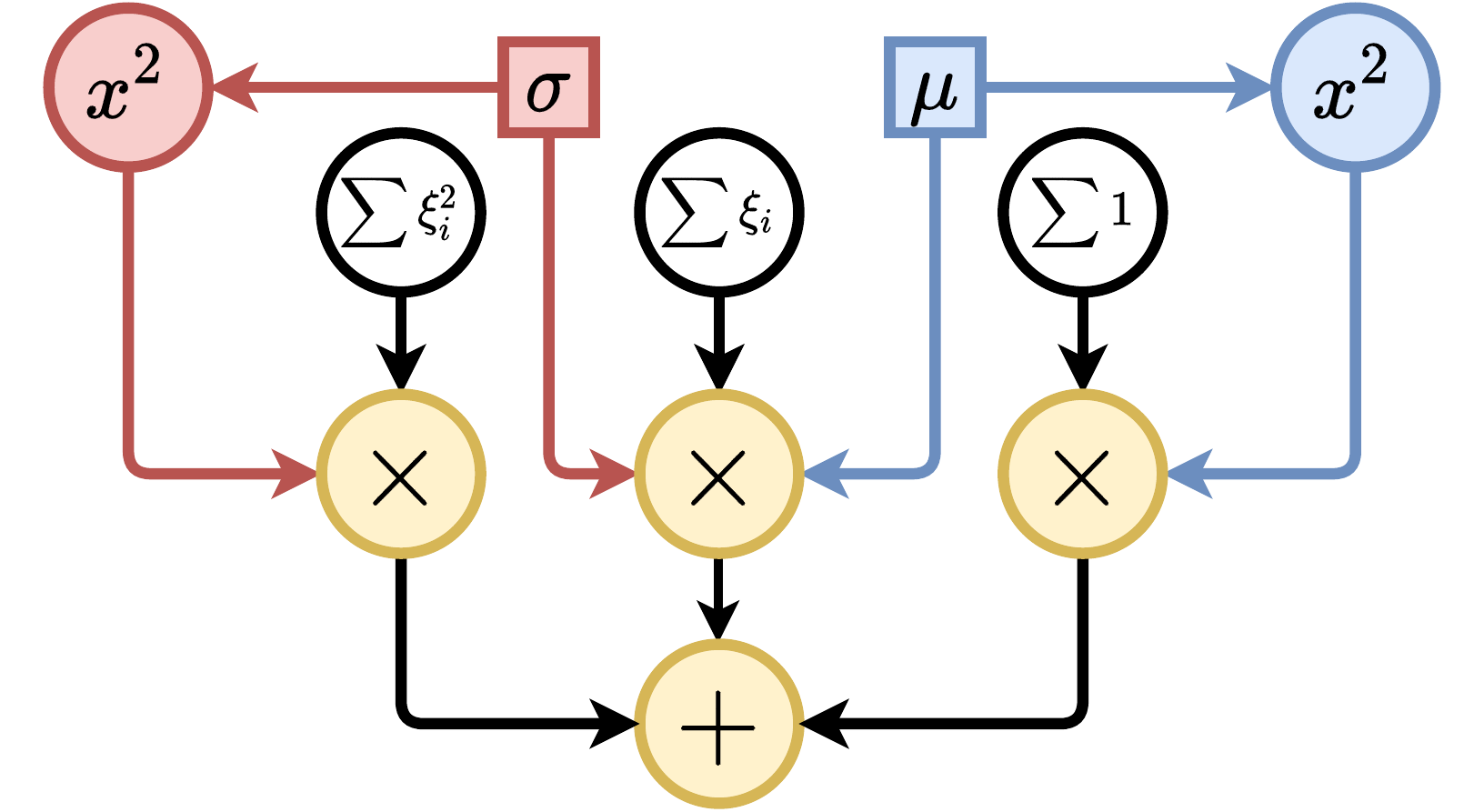}
    \vspace{-2pt}
    \caption{$M = \cdots$}
    \end{subfigure}
    
    \caption{Computation graphs corresponding to MC approximation of $KL$ term. (a)-(c) direct implementation, (d) - our method for any number of MC iterations}
    \label{fig:rad_norm_cg}
\end{figure*}

\section{Interactive application to evaluate MC approximation of $KL$ terms}
\subsection{Example}
To demonstrate the relationship between MC estimation quality of the $KL$ term and number of MC iterations, we provide an interactive Shiny application in this supplement. If one assumes that approximate posterior and prior are Gaussian distributed, in this setting, we can calculate the ``ground truth'' $KL$. The main purpose here is to evaluate MC approximation by calculating the sample variance. The main parameters which will influence the quality of the MC approximation are: number of MC iterations, choice of variance of the approximate posterior distribution, and the size of the model (i.e. number of parameters in Neural Network). All these parameters can be set in our application. In addition, we provide an option to plot results in log-scale (for a better comparison). Additionally, graphs can be zoomed in, by highlighting a selected zone on the plot and double clicking (to zoom out, double click again).

Sample runs of our application (if the reader cannot run the tool) appear in Figure \ref{fig:kl_sim}. We fix variance equal to $10^{-4}$, number of MC iterations up to $10^3$, number of simulations per MC equals to $10^2$ (to smooth the variance estimation). Then we compare the variance of 4 different models with number of parameters: $10^2, 10^4, 10^6$ and $10^8$, and plot the results on the bottom figure with a log-scale. We see that despite the small variance, with growing size of the model it is necessary to increase number of MC iterations to decrease the variance of the MC estimator for $KL$. 

\begin{figure}[h]
    \centering
    \begin{subfigure}[t]{0.19\columnwidth}
    \vskip 0pt
    \includegraphics[width=\textwidth]{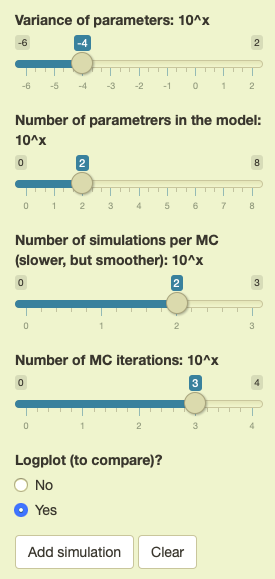}
    \end{subfigure}%
    \begin{subfigure}[t]{0.5\columnwidth}
    \vskip 0pt
    \includegraphics[width=\textwidth]{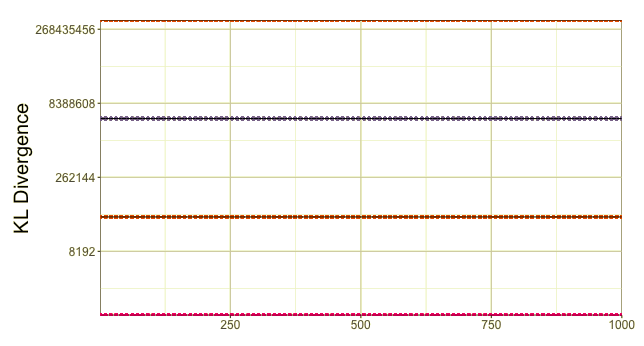}
    \includegraphics[width=\textwidth]{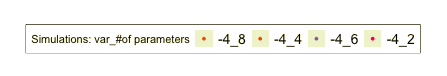}
    \includegraphics[width=\textwidth]{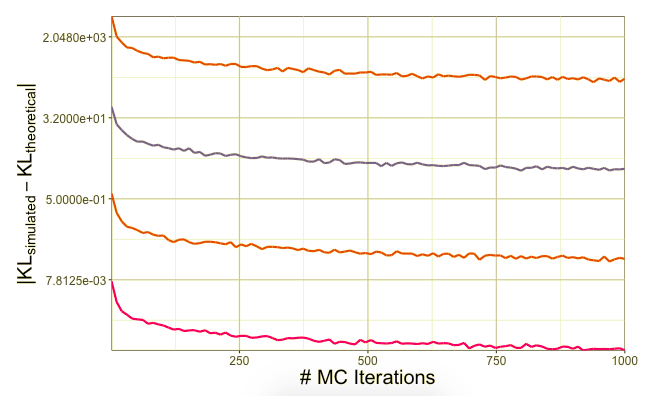}
    \end{subfigure}

    \caption{Demonstration of interactive Shiny application to evaluate the performance of MC estimation.}
    \label{fig:kl_sim}
\end{figure}

\subsection{Installation}
Files are located in the directory "interactive\_app".
There are two ways to prepare our application for execution, both of them are handled by the integer parameter  ``method":
\begin{verbatim}
install_shiny_mc_repar.sh method
\end{verbatim}
``method" can be one of 2 values: 1 or 2 
\begin{enumerate}
 \item If you have R installed, then the following packages are required to be installed and their installation is handled automatically:
 \begin{verbatim}
 c("shiny", "RColorBrewer",
   "dplyr", "ggplot2", "latex2exp") 
 \end{verbatim}
 \item If you would like to avoid installing R, but you have Docker installed, the script creates a Docker image with all necessary dependencies. It will take about 1.8GB of space and can be checked by running 
\begin{verbatim} 
 docker images
\end{verbatim}
\end{enumerate}

\subsection{Execution}
After installation is successful, to run the application,  execute the following script with a new ``method" parameter:
\begin{verbatim}
run_shiny_app.sh method
\end{verbatim}

``method" can be one of 2 values: 1 and 2

\begin{enumerate}
    \item [1. (R route)] It will start app automatically in a browser.
    \item [2. (Docker route)] In this case script starts shiny application and provides a local address, which you can access through the browser (this is address on your local machine, not external: http://localhost:3838)
\end{enumerate}

\subsection{Additional information about version of packages, which were used to run application}
\begin{verbatim}
> sessionInfo()
R version 3.4.2 (2017-09-28)
Platform: x86_64-apple-darwin15.6.0
Running under: macOS  10.14.6

attached packages:
latex2exp_0.4.0
ggplot2_2.2.1
dplyr_0.7.4
RColorBrewer_1.1-2
shiny_1.0.5       
\end{verbatim}

\section{Theory: Proofs and Clarifications}

\subsection{Proofs of Main Results}

\begin{theorem}
If $W(\theta, \xi) = \eta(\theta)T(\xi)$ ($S = 1$), then there exists a parametrization tuple $P$ with $d_P = 1$ for following functions $g(w)$: $w^k$, $\log(w)$, and $\frac{1}{w^k}$.
\end{theorem}
\begin{proof}
First, we are going to show that for $g(w)$: $w^k$, $\log(w)$, and $\frac{1}{w^k}$ there exists a parametrization tuple $P$ of a specific form, and then we show that for these tuples $d_P = 1$.

\begin{compactenum}[{\bf Case} (1).]
\item {\bf $\mathbf{g(w) = w^k}$}
$$
    \begin{aligned}
        \sum_{i=1}^M g(w_i) &=  \sum_{i=1}^M w_i^k
        = \sum_{i=1}^M (\eta(\theta)T(\xi_i))^k \\
        &= \eta^k(\theta)\sum_{i=1}^MT^k(\xi_i)
        = \eta^k(\theta)T^k(\undtil{\xi})\\
        &=n(\theta)t(\xi)
    \end{aligned}
$$
\item {\bf $\mathbf{g(w) = \log(w)}$}
$$
    \begin{aligned}
        \sum_{i=1}^M g(w_i) &=  \sum_{i=1}^M \log(w_i)
        = \sum_{i=1}^M \log(\eta(\theta)T(\xi_i)) \\
        &= M\log(\eta(\theta)) + \sum_{i=1}^M\log(T(\xi_i))\\
        &= M\log(\eta(\theta)) +  \log(T(\undtil{\xi}))\\
        &=n(\theta)t(\xi)
    \end{aligned}
$$
\item {\bf $\mathbf{g(w) = \frac{1}{w^k}}$}
$$
    \begin{aligned}
        \sum_{i=1}^M g(w_i) &=  \sum_{i=1}^M \frac{1}{w_i^k}
        = \sum_{i=1}^M \frac{1}{(\eta(\theta)T(\xi_i))^k} \\
        &= \frac{1}{\eta^k(\theta)}\sum_{i=1}^M\frac{1}{T^k(\xi_i)}\\
        &= n(\theta)t(\xi)
    \end{aligned}
$$
\end{compactenum}
We see that for all $g(w)$ from the list, we identify a parametrization tuple $P=(G(n, t),n(\theta),t(\xi))$, such that $G(n,t) = nt$ and $n(\theta)$, $t(\theta)$ depends on choice of $g(w)$. Clearly, for all these parametrization tuples $P$, $d_P = 1$.
\end{proof}

\begin{theorem}
If $W(\theta, \xi)= \sum_{s=1}^S\eta_s(\theta) T_s(\xi)$, and $g(w) = w^k$, then there exists a parametrization tuple $P$ with $d_P = \binom{k + S -1}{S - 1}$.
\end{theorem}
\begin{proof}
Consider an MC expression $\frac{1}{M}\sum_{i=1}^Mg(w_i)$. Given assumptions on $g(w)=w^k$ and $W=\sum_{s=1}^S\eta_s(\theta) T_s(\xi)$, we get:
$$
        \sum_{i=1}^M g(w_i) =  \sum_{i=1}^M w_i^k = \sum_{i=1}^M \left(\sum_{s=1}^S\eta_s(\theta) T_s(\xi_i)\right)^k 
$$
We observe that $(\sum_{s=1}^S\eta_s(\theta) T_s(\xi_i))^k$ is a polynomial of order k with S indeterminates $T_s(\xi_i)$, and coefficients $\eta_s(\theta)$ independent of $\xi_i$. Let us denote the polynomial as $p_k^S(T(\xi_i); \eta(\theta) )$. Since coefficients are independent of $\xi_i$ for all i, then
$$
        \sum_{i=1}^M p_k^S(T(\xi_i); \eta(\theta)) = p_k^S(T(\undtil{\xi}); \eta(\theta)),
$$
where $\undtil{\xi}=(\xi_1, \ldots, \xi_M)$. 
Which results in the following parametrization tuple $P=(G, n(\theta), t(\xi))$, such that $G(n, t) = \sum_{i=1}^{d_P}n_it_i$, and $n_i$, $t_i$ are coefficients and indeterminates of the polynomial $p_k^S(T(\undtil{\xi}); \eta(\theta))$.\\
The expansion of a polynomial of order $k$ with $S$ indeterminates has $\binom{k+S-1}{S-1}$ coefficients, exactly the number of interactions $d_P$.
\end{proof}

\addtocounter{corollary}{2}
\begin{corollary}
If $W(\theta, \xi) = \sum_{s=1}^S\eta_s(\theta) T_s(\xi)$, and $g(w) = p_K(w)$, then there exists parametrization tuple $P$, such that for any $M$ iterations
\begin{align}
    d_P = \binom{K+S}{S} -1.
\end{align}
\end{corollary}

\textit{Note}.
We consider a polynomial $p_K(w)= \sum_{k=0}^Ka_kw^k$, such that coefficients $a_k$ do not depend on optimized parameter $\theta$. 
Since the first element $a_0$ of the polynomial $p_k(w)$ is not important for the analysis of $d_P$, we can ignore it and compute $d_P$ as presented in Eq.~\ref{form:dp_polyn}. However, if there is a need to consider $a_0$, then one only needs to add $+1$ to the Eq.~\ref{form:dp_polyn}.

\begin{lemma} \label{lemma:cut_size}
Consider a polynomial of order k with $S+1$ indeterminates
$$ \left(\sum_{s=1}^{S}\eta_s(\theta) T_s(\xi) - a\right)^{k},$$ 
where $a$ is a constant. If there $\exists j: T_j(\xi) = c = \text{const}$, then 
$$ \left(\sum_{s=1}^{S}\eta_s(\theta) T_s(\xi) - a\right)^{k}  = \left(\sum_{s=1}^S\eta^*_s(\theta) T_s(\xi)\right)^{k}$$
where  $\eta^*=(\eta_1^*, \ldots, \eta_S^*):$ 
$\forall s \neq j,~ \eta^*_s(\theta) = \eta_s(\theta)$ and $\eta^*_j(\theta) = \eta_j(\theta) - a/c$.
\end{lemma}
\begin{proof}
The proof is direct by substituting  $\eta^*_j(\theta)$.
\end{proof}

\begin{theorem}
Let $W(\theta, \xi) = \sum_{s=1}^S\eta_s(\theta) T_s(X)$, $S \geq 2$. If an approximation of $g(w)$ is made with $K$ Taylor terms, then Corollary~\ref{cor:poly} applies.
\end{theorem}
\begin{proof}
Consider Taylor approximation of $g(w)$, with proper $a$ and $K$ terms:
$$g(w) = \sum _{k=0}^K {\frac {g^{(k)}(a)}{k!}}(w-a)^{k}, \text{ where } K \text{ can be } \infty.$$
Then,
\begin{equation}
\begin{aligned}[b]
        \sum_{i=1}^M g(w_i) &=  \sum_{i=1}^M g(w_i) \\
        &=\sum_{i=1}^M \sum _{k=0}^{K} c_k (w_i-a)^k\\
        &=\sum _{k=0}^{K}\sum_{i=1}^M c_k(w_i-a)^{k}\\
        &=\sum _{k=0}^{K}\sum_{i=1}^M c_k\left(\sum_{s=1}^S\eta_s(\theta) T_s(\xi_i) - a\right)^{k}
        \label{form:tayl_polyn}
\end{aligned}
\end{equation}
From this point there are 2 ways to apply Corollary~\ref{cor:poly}:
\paragraph{1. Polynomial of order $K$ with $S+1$ indeterminates}~\\
Eq.~\eqref{form:tayl_polyn} can be considered as a polynomial of order $K$ with $S+1$ indeterminates. Then according to Corollary~\ref{cor:poly}, for corresponding parametrization tuple $P$, $d_P= \frac{(K+1)\binom{K+S+1}{S}}{S+1} - 1.$ However, since $a$ in Eq.~\ref{form:tayl_polyn} is constant, then $p_K(w-a)$ has $K$ terms depending just on $a$ and can be disregarded in the computation graph, since there is no interaction with parameters. This results in 
$$
d_P= \frac{(K+1)\binom{K+S+1}{S}}{S+1} - (K+1)
$$
\paragraph{2. Polynomial of order K with S indeterminates}~\\
In some cases Eq.~\eqref{form:tayl_polyn} can be considered as polynomial of order $K$ with $S$ new indeterminates, following Lemma~\ref{lemma:cut_size}. Then according to Corollary~\ref{cor:poly}, in addition to a node responsible for reparameterization of $\eta^*(\theta)$, we get
$$
d_P= \frac{(K+1)\binom{K+S}{S-1}}{S} +1
$$
which reduces to the form in Eq.~\eqref{form:dp_polyn}.
\end{proof}

\end{document}